\pdfoutput=1
\documentclass[11pt]{article}

\usepackage{emnlp2022}

\usepackage{aflt}

\usepackage{times}
\usepackage{latexsym}

\newcommand{\Parens}[1]{\ensuremath{\left(#1\right)}}

\usepackage[T1]{fontenc}
\usepackage[utf8]{inputenc}
\usepackage{microtype}
\usepackage{inconsolata}

\newcommand{\compactminus}{\mathord-}
\newcommand{\compactplus}{\mathord+}

\usepackage{pgfplots}
\pgfplotsset{compat=1.15}
\definecolor{myblue}{RGB}{0,58,109}
\definecolor{myred}{RGB}{175,0,0}

\usepackage{newtxmath}

\title{Algorithms for Weighted Pushdown Automata}

\newcommand{\mailto}[2]{\texttt{\href{mailto:#1}{#2}}}

\author{Alexandra Butoi$^1$ \quad
Brian DuSell$^2$ \quad
Tim Vieira$^3$ \\
\textbf{Ryan Cotterell$^1$ \quad
David Chiang$^2$} \\[1ex]
  $^1$ETH Zürich \quad
  $^2$University of Notre Dame \quad
  $^3$Johns Hopkins University \\[1ex]
    \mailto{alexandra.butoi@inf.ethz.ch}{alexandra.butoi@inf.ethz.ch} \quad
    \{\mailto{bdusell1@nd.edu}{bdusell1}, 
    \mailto{dchiang@nd.edu}{dchiang}\}\texttt{@nd.edu} \\\
    \{\mailto{tim.f.vieira@gmail.com}{tim.f.vieira},
    \mailto{ryan.cotterell@gmail.com}{ryan.cotterell}\}\texttt{@gmail.com}
}

\begin{document}
\maketitle

\begin{abstract}
Weighted pushdown automata (WPDAs) are at the core of many natural language processing tasks, like syntax-based statistical machine translation and transition-based dependency parsing. As most existing dynamic programming algorithms are designed for context-free grammars (CFGs), algorithms for PDAs often resort to a PDA-to-CFG conversion. In this paper, we develop novel algorithms that operate directly on WPDAs. Our algorithms are inspired by Lang's algorithm, but use a more general definition of pushdown automaton and either reduce the space requirements by a factor of $|\stackalphabet|$ (the size of the stack alphabet) or reduce the runtime by a factor of more than $|\states|$ (the number of states). When run on the same class of PDAs as Lang's algorithm, our algorithm is both more space-efficient by a factor of $|\stackalphabet|$ and more time-efficient by a factor of $|\states| \cdot |\stackalphabet|$.

\vspace{0.5em}

\centering
\mbox{\includegraphics[width=1.25em,height=1.25em]{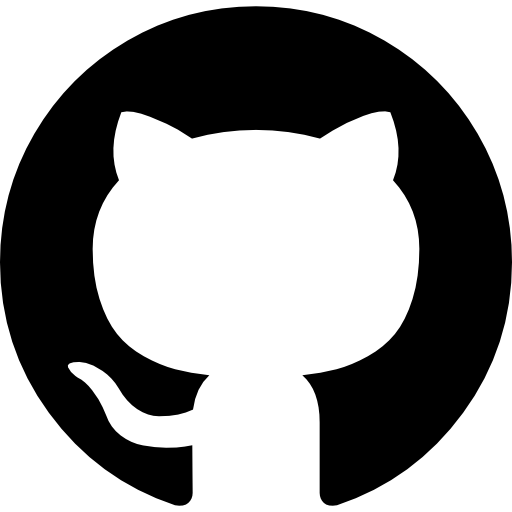}\hspace{0.5em}%\parbox{\dimexpr\linewidth-2\fboxsep-2\fboxrule}
\raisebox{0.3em}{\url{https://github.com/rycolab/wpda}}}
\end{abstract}

\section{Introduction}

Weighted pushdown automata (WPDAs) are widespread in natural language processing (NLP), primarily in syntactic analysis. 
For instance, WPDAs have found use in syntax-based statistical machine translation \citep{allauzen-etal-2014-pushdown}, and many transition-based dependency parsers \cite{nivre-2004-incrementality, chen-manning-2014-fast, weiss-etal-2015-structured, dyer-etal-2015-transition, andor-etal-2016-globally, shi-etal-2017-fast, ma-etal-2018-stack, fernandez-gonzalez-gomez-rodriguez-2019-left} are special cases of WPDAs.
In addition, PDAs have been used in computational psycholinguistics as models of human sentence processing \citep{resnik-1992-left}.
Despite their ubiquity, there has been relatively little research on the theory of WPDAs themselves.
In some ways, WPDAs are treated as second-class citizens compared to their equivalent %\footnote{WPDAs and WCFGs are strongly equivalent formalisms in that for any WPDA there exists a WCFG for which there is a yield-preserving bijection between their derivation language and vice versa.}
cousins, weighted context-free grammars (WCFGs), for which a variety of dynamic programming algorithms exists \cite{bar-hillel-1961-formal,earley-1970-efficient, stolcke-1995-efficient}.
To help fill this gap, this paper offers several new and improved algorithms for computing with WPDAs.

\begin{figure}
\centering
\begin{tikzpicture}[x=2.5cm,y=1.75cm]
\node[align=center](extpda) at (0,0) {Extended\\PDA};
\node[align=center](topdown) at (0.5,1) {Top-down \\ PDA};
\node[align=center](bottomup) at (0.5,-1) {Bottom-up \\ PDA};
\node[align=center](topdownnf) at (1.6,1) {Top-down \\ NF PDA};
\node[align=center](lang) at (1.6,0) {\Simple{} \\ PDA};
\node[align=center](bottomupnf) at (1.6,-1) {Bottom-up \\ NF PDA};
\node(stringsum) at (2.4,0) {Stringsum};
\node(cfg) at (0.5,1.75) {CFG};
\node(cnf) at (1.6,1.75) {Chomsky NF};
\draw[->,bend right=10] (extpda) edge node[sloped,below,pos=0.4] {\tiny\cref{thm:pda_to_topdown}} (topdown);
\draw[->,dashed,bend left=10] (extpda) edge node[sloped,above,pos=0.4] {\tiny Aho-Ullman} (topdown);
\draw[->] (extpda) -- node[sloped,below,pos=0.4] {\tiny\cref{thm:pda_to_bottomup}} (bottomup);
\draw[->] (topdown) -- node[sloped,above] {\tiny\cref{thm:topdown_nf}} (topdownnf);
\draw[->] (bottomup) -- node[sloped,below] {\tiny\cref{thm:bottomup_nf}} (bottomupnf);
\draw[->] (topdownnf) -- node[sloped,above]{\tiny\cref{algo:fast-top-down-parsing}} (stringsum);
\draw[->,dashed] (lang) -- node[above]{\tiny Lang} (stringsum);
\draw[->] (bottomupnf) -- node[sloped,below] {\tiny\cref{algo:fast-bottom-up-parsing}} (stringsum);
\draw[->,dashed] (topdown) -- (cfg);
\draw[->,dashed] (cfg) -- (cnf);
\draw[->,dashed] (topdownnf) -- (cnf);
\draw[->,dashed,bend left] (cnf) edge node[sloped,above]{\tiny CKY} (stringsum);
\draw[draw=none] (lang) -- node[sloped]{$\subset$} (topdownnf);
\draw[draw=none] (lang) -- node[sloped]{$\subset$} (bottomupnf);
\end{tikzpicture}
\caption{Roadmap of the paper. Solid lines are new results in this paper; dashed lines are old results. We are aware of two existing methods for PDA stringsums, via CFG and via Lang's algorithm; our algorithms are faster and/or more general than both.}
\label{fig:overview}
\end{figure}

% discuss why we defined things the way we did
Figure~\ref{fig:overview} gives an overview of most of our results.
We start by defining a weighted version of the extended PDAs of \citet[p.~173]{aho+ullman:1972} and two special cases: the standard definition \citep{hopcroft-2006-introduction}, which we call top-down, and its mirror image, which we call bottom-up.
Both top-down and bottom-up WPDAs have been used in NLP\@.
\citeposs{roark-2001-probabilistic} generative parser is a top-down PDA as is \citeposs{dyer-etal-2016-recurrent}.
%Furthermore, \citet{abney-etal-1999-relating} show that top-down WPDAs admit an efficient weight-pushing algorithm, which makes them easier to use as language models.
Most transition-based dependency parsers, both arc-standard \cite{nivre-2004-incrementality, huang-etal-2009-bilingually} and arc-eager \cite{nivre-2003-efficient, zhang-clark-2008-tale}, are bottom-up WPDAs.
%The distinction between top-down and bottom-up WPDAs is important for algorithm development.

Next, we give a normal form for WPDAs analogous to Chomsky normal form,
and we derive new dynamic programming algorithms to compute the weight of a string under top-down and bottom-up WPDAs in normal form.
We are only aware of one previous recognition algorithm for PDAs, that of \citet{lang-1974-deterministic}, which we generalize to the weighted case and improve in the following ways:
\begin{compactitem}
\item On PDAs more general than those Lang considers,
our algorithm is more space-efficient by a factor of $|\stackalphabet|$ (the stack alphabet size);
\item We can speed up our algorithm to be more time-efficient by a factor of more than $|\states|$ (the number of states), but without the space-complexity improvement;
\item On the same PDAs that Lang considers, which we call \defn{\simple}, our sped-up algorithm is more efficient by a factor of $|\stackalphabet|$ in space and  $|\states| \cdot |\stackalphabet|$ in time.
\end{compactitem}
Compared with the pipeline of standard procedures for converting a top-down PDA to a CFG, converting to Chomsky normal form, and parsing with CKY, our top-down algorithm is faster by a factor of more than $O(|\states|^3)$.

Finally, we present iterative algorithms for computing the total weight of all runs of a WPDA\@.

%Overall, we give the first unified treatment of various algorithms for WPDAs \cite{kuhlmann-etal-2011-dynamic, shi-etal-2017-fast, buys-blunsom-2018-neural, gomez-rodriguez-etal-2018-global}. 
%Implementations of the algorithms presented in this paper are freely available online at  \url{www.anonymous.com}.\looseness=-1

%and a speedup. These improve on Lang's algorithm in several ways:
%of the paper is a study of algorithms that operate directly on PDAs.
%While some previous algorithms use bounded stack depth \citep[e.g.][]{schuler-2009-positive,Alluzen2014} or greedy search \citep[e.g.][]{???}, we're interested here in exact algorithms.

% discuss new algorithms
% \timv{Other bits of motivation: transition systems for parsing (especially dependency parsing) are very popular (even in the neural network age; can cite Tianze's papers for recent good results). Devising exact DP algorithms for new systems is tedious, but essentially amount to a PDA-to-CFG conversion.  We give the general construction here.  Our construction is better than the best-known general construction in the literature.  One thing we could do is apply the construction to some transition systems and compare the relative sizes of the conversions.  We could also comment on how our automatic conversion compares to handcrafted conversions in some of the Gomez-Rodriguez and company papers. }

\section{Weighted Pushdown Automata}\label{sec:definitions}

\subsection{Preliminaries}

Let $\range{i}{j}$ denote the sequence of integers $(i, \ldots, j)$.
If $\str$ is a string, we write $|\str|$ for the length of $\str$, 
$\str_i$ for the $i^{\text{th}}$ symbol of $\str$, and $\str\slice{i}{j}$ for the substring $\str_{i+1} \cdots \str_{j}$.

\begin{defin}
A \defn{monoid} is a tuple $(\semiringset, \odot, \id)$, where $\semiringset$ is a set, $\odot$ is an associative binary operation, and $\id \in \semiringset$, called the \defn{identity} element, satisfies $\id \odot a = a \odot \id = a$ for all $a \in \semiringset$. If $a \odot b = b \odot a$ for all $a,b$, we say that the monoid is \defn{commutative}. 
\end{defin}

\begin{defin}
A \defn{semiring} is a tuple $\semiring = \semiringtuple$ 
%in which the set $\semiringset$ is equipped with two binary operations $\oplus$ and $\otimes$ 
such that $\left(\semiringset, \oplus, \zero \right)$ is a commutative monoid and $\left(\semiringset, \otimes, \one \right)$ is a monoid. Additionally, $\otimes$ distributes over $\oplus$, that is, $a \otimes (b \oplus c) = a \otimes b \oplus a \otimes c$ and $(a \oplus b) \otimes c = a \otimes c \oplus b \otimes c$, and $\zero$ is absorbing with respect to $\otimes$, that is, $\zero \otimes a = a \otimes \zero = \zero$. If $\otimes$ is commutative then we say that $\semiring$ is \defn{commutative}.
\end{defin}

We also sometimes assume $\semiring$ is \defn{continuous}; please see the survey by \citet{droste-2009-semirings} for a definition.

\subsection{Definition}

Our definition of weighted PDA is more general than usual definitions, in order to accommodate the top-down and bottom-up variants introduced in \cref{sec:tdbu}. It is roughly a weighted version of extended PDAs of \citet[p.~173]{aho+ullman:1972} and the PDAs of \citet[p.~131]{lewis+papadimitriou:1997}.

\begin{defin}
A \defn{weighted pushdown automaton} (WPDA) over a semiring $\semiring=\semiringtuple$ is a tuple $\pushdown = (\states, \alphabet, \stackalphabet, \trans, \pdaconfig{\initstack}{\startstate}, \pdaconfig{\finalstack}{\finalstate})$, where:
\begin{compactitem}
    \item $\states$ is a finite set of states;
    \item $\alphabet$ is a finite set of input symbols, called the input alphabet;
    \item $\stackalphabet$ is a finite set of stack symbols, called the stack alphabet;
    \item $\trans \colon \states \times \stackalphabet^* \times \left( \alphabet \cup  \set{\varepsilon} \right) \times \states \times \stackalphabet^* \rightarrow \semiringset$ is called the transition weighting function;
    \item $\pdaconfig{\initstack}{\startstate}$ is called the initial configuration and $\pdaconfig{\finalstack}{\finalstate}$ is called the final configuration, where $\startstate, \finalstate \in \states$ and $\initstack, \finalstack \in \stackalphabet^*$ .
% This is more general, but as written, too general, because it allows the PDA to recognize arbitrary languages.    
%    \item $\initf \colon \states \times \stackalphabet^* \rightarrow \semiringset$ is called the initial weighting function.
%    \item $\finalf \colon \states \times \stackalphabet^* \rightarrow \semiringset$ is called the final weighting function.
\end{compactitem}
\end{defin}

Stacks are represented as strings over $\stackalphabet$, from bottom to top. Thus, in the stack $\stackseq = \stacksymbol{X_1} \stacksymbol{X_2} \cdots \stacksymbol{X_n}$, the symbol $\stacksymbol{X_1}$ is at the bottom of the stack, while $\stacksymbol{X_n}$ is at the top.

\begin{defin}
A \defn{configuration} of a WPDA is a pair
$\pdaconfig{\stackseq}{q}$, where $q \in \states$ is the current state and ${\stackseq} \in \stackalphabet^*$ is the current contents of the stack.
\end{defin}

The initial and final configurations of a WPDA are examples of configurations; it is possible to generalize the initial and final stacks to (say) regular expressions over $\stackalphabet$, but the above definition suffices for our purposes.

A WPDA moves from configuration to configuration by following transitions of the form $\pdaEdge{w}{q}{a}{r}{\stackseq_1}{\stackseq_2}$, which represents a move from the state $q$ to state $r$, while popping the sequence of symbols $\stackseq_1 \in \stackalphabet^*$ from the top of the stack and pushing the sequence $\stackseq_2 \in \stackalphabet^*$.

\begin{defin}
If $\trans(p,\stackseq_1,a,q,\stackseq_2) = w$, then we usually write $\transweight{p}{a}{q}{\stackseq_1}{\stackseq_2} = w$ or that $\delta$ has transition $(\pdaEdge{w}{q}{a}{p}{\stackseq_1}{\stackseq_2})$.
We sometimes let $\atrans$ stand for a transition, and we define $\trans(\atrans) = w$.
We say that $\atrans$ \defn{scans} $a$, and if $a \neq \varepsilon$, we call $\atrans$ \defn{scanning}; otherwise, we call it \defn{non-scanning}.
We say that $\atrans$ is $k$-pop, $l$-push if $|{\stackseq_1}|=k$ and $|{\stackseq_2}|=l$.
\end{defin}

\begin{defin}
If $\pdaconfig{\stackseq \stackseq_1}{q_1}$ and $\pdaconfig{\stackseq \stackseq_2}{q_2}$ are configurations, 
and $\atrans$ is a transition $\pdaEdge{w}{q_1}{a}{q_2}{\stackseq_1}{\stackseq_2}$,
we write $\pdaconfig{\stackseq \stackseq_1}{q_1} \Rightarrow_\atrans \pdaconfig{\stackseq \stackseq_2}{q_2}$.
\end{defin}

\begin{defin}
A \defn{run} of a WPDA $\pushdown$ is a sequence of configurations and transitions \[\arun=\pdaconfig{\stackseq_0}{q_0}, \atrans_1, \pdaconfig{\stackseq_1}{q_1}, \ldots, \atrans_n, \pdaconfig{\stackseq_n}{q_n}\] where, for $i = 1, \ldots, n$, we have $\pdaconfig{\stackseq_{i-1}}{q_{i-1}} \Rightarrow_{\atrans_i} \pdaconfig{\stackseq_i}{q_i}$.
(Sometimes it will be convenient to treat $\arun$ as a sequence of only configurations or only transitions.) 
%\brian{It might be worth pointing out that the $\tau$'s are necessary for disambiguating some cases, such as $(q, \texttt{ab}), (\pdaEdge{w}{q}{a}{r}{\texttt{ab}}{\texttt{ac}}), (r, \texttt{ac})$ vs. $(q, \texttt{ab}), (\pdaEdge{w}{q}{a}{r}{\texttt{b}}{\texttt{c}}), (r, \texttt{ac})$. However, they are never necessary for top-down or bottom-up PDAs. \response{david} that's a good point; this definition is annoyingly redundant, and it's more annoying that we never actually make use of it. It's true, though, that sometimes it's more convenient to think about the configs and sometimes the transitions.}
%We write $\prevq{\arun} = \pdaconfig{\stackseq_0}{q_0}$ and $\nextq{\arun} = \pdaconfig{\stackseq_n}{q_n}$.
A run 
%$\arun$ 
is called \defn{accepting} if 
%$\prevq{\arun}$ 
$\pdaconfig{\stackseq_0}{q_0}$
is the initial configuration and %$\nextq{\arun}$ 
$\pdaconfig{\stackseq_n}{q_n}$
is the final configuration.
If, for $i=1, \ldots, n$, $\atrans_i$ scans $a_i$, then we say that $\arun$ scans the string $a_1 \cdots a_n$.
We write $\runs \left(\pushdown, \str \right)$ for the set of runs that scan $\str$ and $\runs \left( \pushdown \right)$ for the set of all accepting runs of $\pushdown$.
\end{defin}

% \begin{proposition}
% Any WPDA is equivalent to one with exactly one initial state $s$ and exactly one final state $f$. 
% \end{proposition}
% \begin{proof}
% We can replace a set $\init$ of initial states with an extra state $s$ and add the transitions $\pdaEdge{\one}{s}{\varepsilon}{q_i}{\varepsilon}{\varepsilon}$ . Similarly, any WPDA with final states $\final$ can be converted into an equivalent one with a single final state by adding an extra state $f$ and the transitions $\pdaEdge{\one}{q_f}{\varepsilon}{f}{\varepsilon}{\varepsilon}$ for each $q_f \in \final$.
% \end{proof}

\subsection{Subclasses of PDAs}
\label{sec:tdbu}

Next, we define two special forms for WPDAs, which we call \defn{top-down} and \defn{bottom-up}, because they can be used as top-down and bottom-up parsers for CFGs, respectively. The most common definition of PDA \citep{hopcroft-2006-introduction,autebert+:1997} corresponds to top-down PDAs,\footnote{This definition goes back to \citeposs{chomsky:1963} original definition, which also allows 0-pop transitions.} while the machine used in an LR parser \citep{knuth-1965-translation} corresponds to bottom-up PDAs.

\begin{defin}
A WPDA is called \defn{bottom-up} if it has only 1-push transitions. 
%Moreover, the initial and final weighting functions satisfy the constraints:
%\begin{align*}
%    \initf{q, \stackseq} &= \zero && \text{if } \stackseq \neq \varepsilon \\
%    \finalf{q, \stackseq} &= \zero && \text{if } \stackseq \neq \stacksymbol{S}
%\end{align*}
%for some $q \in \states$ and $\stacksymbol{S} \in \stackalphabet$.
Moreover, the initial configuration is $\pdaconfig{\varepsilon}{\startstate}$ and the final configuration is $\pdaconfig{\stacksymbol{S}}{\finalstate}$ for some $\startstate, \finalstate \in \states$ and $\stacksymbol{S} \in \stackalphabet$.
\end{defin}

\begin{proposition} \label{thm:pda_to_bottomup}
Every WPDA is equivalent to some bottom-up WPDA\@.
\end{proposition}
\begin{proof}
Add states $\startstate', \finalstate'$ and a stack symbol $\stacksymbol{S'}$, and make $\pdaconfig{\varepsilon}{\startstate'}$ and $\pdaconfig{\stacksymbol{S'}}{\finalstate'}$ the new initial and final configurations, respectively. Add transitions 
\begin{align*}
\pdaEdgeAlign{\one}{\startstate'}{\varepsilon}{\startstate}{\varepsilon}{\stacksymbol{S'}\stackseq_\init} \\ \pdaEdgeAlign{\one}{\finalstate}{\varepsilon}{\finalstate'}{\stacksymbol{S'}\stackseq_\final}{\stacksymbol{S'}}.
\end{align*}
For each $k$-pop, $l$-push transition $\pdaEdge{w}{p}{a}{r}{\stackseq}{X_1 \cdots X_l}$ where $l>1$, create $(l-1)$ new states $q_1, \dots, q_{l-1}$ and replace the transition with
\begin{align*}
\pdaEdgeAlign{\one}{p}{\varepsilon}{q_1}{\varepsilon}{X_1} \\
\pdaEdgeAlign{\one}{q_{i-1}}{\varepsilon}{q_i}{\varepsilon}{X_i} \quad i = 2, \ldots, l-1 \\
\pdaEdgeAlign{w}{q_{k-1}}{a}{r}{\stackseq}{X_k}.
\end{align*}
For each $k$-pop, $0$-push transition $\pdaEdge{w}{q}{a}{p}{\stackseq}{\varepsilon}$, replace it with the $(k+1)$-pop, $1$-push transitions 
$\pdaEdge{w}{q}{a}{p}{\stacksymbol{X} \stackseq}{\stacksymbol{X}}$
for every $\stacksymbol{X} \in \stackalphabet \cup \{\stacksymbol{S'}\}$.
\end{proof}

If the original WPDA had transitions that push at most $l$ symbols, the resulting WPDA has $O(l \cdot |\delta| \cdot |\states|)$ states and $O((l + |\stackalphabet|) \cdot |\delta|)$ transitions.

\begin{defin}
A WPDA is called \defn{top-down} if it has only 1-pop transitions. Moreover, the initial configuration is $\pdaconfig{\stacksymbol{S}}{\startstate}$ and the final configuration is $\pdaconfig{\varepsilon}{\finalstate}$
for some $\startstate, \finalstate \in \states$ and $\stacksymbol{S} \in \stackalphabet$.
%Moreover, the initial and final weighting functions satisfy the constraints
%\begin{align*}
%    \initf{q, \stackseq} &= \zero && \text{if } \stackseq \neq \stacksymbol{S} \\
%    \finalf{q, \stackseq} &= \zero && \text{if } \stackseq \neq \varepsilon
%\end{align*}
%for some $q \in \states$ and $\stacksymbol{S} \in \stackalphabet$.
\end{defin}

\begin{proposition} \label{thm:pda_to_topdown}
Every WPDA is equivalent to some top-down WPDA\@.
\end{proposition}

\begin{proof}
Similar to the bottom-up case.
\iffalse
Add states $\startstate', \finalstate'$ and a stack symbol $\stacksymbol{S'}$, and make $\pdaconfig{\stacksymbol{S'}}{\startstate'}$ and $\pdaconfig{\varepsilon}{\finalstate'}$ the new initial and final configurations, respectively. Add transitions 
\begin{align*}
\pdaEdgeAlign{\one}{\startstate'}{\varepsilon}{\startstate}{\stacksymbol{S'}}{\stacksymbol{S'}\stackseq_\init} \\ \pdaEdgeAlign{\one}{\finalstate}{\varepsilon}{\finalstate'}{\stacksymbol{S'}\stackseq_\final}{\varepsilon}.
\end{align*}
%
For each $k$-pop, $l$-push transition $\pdaEdge{w}{p}{a}{r}{X_1\cdots X_k}{\stackseq}$ where $k>1$, create $(k-1)$ new states $q_1, \dots, q_{k-1}$ and replace the transition with
\begin{align*}
\pdaEdgeAlign{\one}{p}{\varepsilon}{q_1}{X_1}{\varepsilon} \\
\pdaEdgeAlign{\one}{q_{i-1}}{\varepsilon}{q_i}{X_i}{\varepsilon} \quad i = 2, \ldots, k-1 \\
\pdaEdgeAlign{w}{q_{k-1}}{a}{r}{X_k}{\stackseq}.
\end{align*}
%
For each $0$-pop, $l$-push transition $\pdaEdge{w}{q}{a}{p}{\varepsilon}{\stackseq}$, replace it with the $1$-pop, $(l+1)$-push transitions $\pdaEdge{w}{q}{a}{p}{\stacksymbol{X}}{\stacksymbol{X} \stackseq}$ for every $\stacksymbol{X} \in \stackalphabet \cup \{\stacksymbol{S'}\}$, which pop the top stack symbol and then push it back to the stack together with the sequence~$\stackseq$.
\fi
\end{proof}
This conversion crucially makes use of nondeterminism to guess the top $k$ stack symbols.
\citet[p.~174]{aho+ullman:1972} give a different algorithm that uses the state to keep track of the top $k$ stack symbols. Although this does not require nondeterminism, it creates $O(|\stackalphabet|^k \cdot |\states|)$ states.

Finally, \citet{lang-1974-deterministic} considers a still more restricted subclass of PDAs.\footnote{This definition is also used by \citet{sipser:2012} and seems to go back to \citeposs{evey:1963} original definition of PDAs, which doesn't allow 1-pop, 1-push transitions. The term ``simple PDA'' has been used at least twice for two different purposes \citep{schutzenberger:1963,lewis+papadimitriou:1997}; we apologize for introducing a third.}
\begin{defin}
A WPDA is called \defn{\simple} if it only has $k$-pop, $l$-push transitions for $k \leq 1$ and $l \leq 1$.
\end{defin}
Because simple PDAs do not condition pushes on the top stack symbol, they can be weighted, but not probabilistic.

\subsection{Stringsums and \Runsum{}s} 
% \timv{It is tempting to only give the algorithm for stringsum (on a lattice);
% since runsum is equal to stringsum(Sigma*).  
% This is analogous to the WCFG case.
% The only reason not to do it is that the stringsum algorithm is more detailed so it might be nice to see the simpler thing first (runsum doesn't have IJK).
% }

\begin{defin}
The \defn{weight} $\normalizer (\arun)$ of a run $\arun \in \runs \left( \pushdown \right)$ is the product of the weights of its transitions,
\begin{equation*}
    \normalizer (\arun) \defeq \bigotimes_{\atrans \in \arun} \trans \Parens{\atrans}.
\end{equation*}
\end{defin}

\begin{defin}
The \defn{stringsum} $\normalizer (\pushdown, \str)$ of a string $\str$ for a WPDA $\pushdown$ is the total weight of all runs of $\pushdown$ that scan $\str$, \begin{equation*}
    \normalizer (\pushdown, \str) \defeq \bigoplus_{\mathclap{\arun\in \runs \left(\pushdown, \str \right)}}
    %\initf \left( p \left( \arun \right) \right) \otimes 
    \normalizer \left( \arun \right).
    %\otimes \finalf \left( n \left( \arun \right) \right)
\end{equation*}
\end{defin}

\begin{defin} \label{def:allsum}
The \defn{\runsum} $\normalizer (\pushdown)$ of a WPDA $\pushdown$ is  the total weight of all runs of $\pushdown$,
\begin{equation*}
    \normalizer ({\pushdown}) \defeq \bigoplus_{\mathclap{\arun\in \runs \left({\pushdown} \right)}} 
    %\initf \left( p \left( \arun \right) \right) \otimes 
    \normalizer(\arun) 
    %\otimes \finalf \left( n \left( \arun \right) \right)
    .
\end{equation*}
\end{defin}

\subsection{Push and Pop Computations}

\label{sec:push-computation}

Our algorithms for bottom-up WPDAs make heavy use of \defn{push computations}.
Intuitively, a push computation is a run that pushes exactly one symbol without touching the stack symbols below it.

\begin{defin}[Push computation] \label{def:push-computation}
Let $\pushdown$ be a bottom-up WPDA and $\str \in \alphabet^*$ an input string. 
A \defn{push computation} of type $\myitem{i}{p}{X}{j}{q}$, where $0 \le i \le j \le |\str|$, $p, q \in \states$, and $X \in \stackalphabet$, is a run $\arun = \pdaconfig{\stackseq_0}{q_0}, \ldots, \pdaconfig{\stackseq_m}{q_m}$ that scans $\str\slice{i}{j}$, where $\stackseq_m = \stackseq_0 X$, $q_0 = p$, $q_m = q$, and for all $l>0$, $|\stackseq_l| \geq |\stackseq_m|$.
\end{defin}

\begin{figure}
    \centering
    \begin{tikzpicture}[x=0.9cm,
    decoration = {snake, pre length=2pt,post length=4pt}]
    \small
        \draw[style={thick, fill=gray!10}] (0,2) rectangle (0.5,3);
        \draw[style={thick, fill=gray!10}] (2,2) rectangle (2.5,3);
        \draw[style={thick}] (2,3) rectangle (2.5,3.5);
        \draw[style={thick, fill=gray!10}] (5,2) rectangle (5.5,3);
        \draw[style={thick}] (5,3) rectangle (5.5,3.5);
        \draw[style={thick, fill=gray!10}] (5,3.5) rectangle (5.5,4.25);
        \draw[style={thick}] (5,4.25) rectangle (5.5,4.75);
        \draw[style={thick, fill=gray!10}] (7,2) rectangle (7.5,3);
        \draw[style={thick}] (7,3) rectangle (7.5,3.5);
        
        \node[draw=none] (p1) at (0.25, 2.5) {$\stackseq$};
        \node[draw=none] (p2) at (2.25, 2.5) {$\stackseq$};
        \node[draw=none] (p3) at (5.25, 2.5) {$\stackseq$};
        \node[draw=none] (p4) at (7.25, 2.5) {$\stackseq$};
        
        \node[draw=none] (p5) at (2.25, 3.25) {$Y_1$};
        \node[draw=none] (p6) at (5.25, 3.25) {$Y_1$};
        \node[draw=none] (p7) at (5.25, 4.5) {$Y_m$};
        \node[draw=none] (p8) at (7.25, 3.25) {$X$};
        \node[draw=none] (p9) at (5.25, 4) {$\vdots$};

        \node[draw=none] (p13) at (3.75, 2.5) {$\cdots$};
        
        \coordinate (c1) at (5.25, 5);
        \coordinate (c2) at (7.25, 3.75);
        
        \begin{scope}[yshift=1.25cm]
        \node[state] (q0) at (0.25,0) { $q_0$ };
        \node[state] (q1) at (2.25,0) { $q_1$ };
        \node[state] (qk) at (5.25,0) { $q_{m-1}$ };
        \node[state] (qk1) at (7.25,0) { $q_{m}$ };
        \end{scope}
        
        \begin{scope}[auto=right,outer sep=4pt]
        \draw (q0) edge[decorate] node {$\strut \str\slice{i_0}{i_1}$} (q1);
        \draw (q1) edge[decorate] node {$\strut \str\slice{i_1}{i_{m-1}}$} (qk);
        \draw (qk) edge node {$\strut a$} (qk1);
        \end{scope}
        
        \draw[->] (c1) to[bend left, out=90, in=120] node[draw=none,pos=0.2,above] {$\pdaEdge{w}{q_{m-1}}{a}{q_{m}}{Y_1 \cdots Y_m}{X}$} (c2);
    \end{tikzpicture}
\caption{
%Example of a push computation of type $\myitem{i_0}{q_0}{X}{q_{k+1}}{i_{k}+|a|}$. 
A push computation is a sequence of transitions that pushes exactly one symbol ($\stacksymbol{X}$) without touching the stack symbols below ($\stackseq$).
The curly edges indicate sequences of transitions (which are themselves push computations) while the straight edge indicates a single transition.}
\label{fig:push-computation}
\end{figure}

\Cref{fig:push-computation} shows an example of a push computation. Notice that this \emph{push} of $X$ might be the result of possibly many transitions that can manipulate the stack. Every symbol other than $X$ that is pushed onto the stack during this computation must be popped later by another transition.
% There may be multiple push computations with the same type $\myitem{i}{p}{X}{j}{q}$; we can think of the item $\myitem{i}{p}{X}{j}{q}$ as standing for all such runs, and its weight $\itemweight{i}{p}{X}{j}{q}$ is the total weight of all such runs. 

The mirror image of a push computation is a \defn{pop computation}, used in algorithms for top-down WPDAs; we defer its definition to \cref{sec:algo-top-down}.

\section{Normal Form}
\label{sec:normal-form}

In this section we present a series of semantics-preserving transformations for converting an arbitrary pushdown automaton into a normal form that is analogous to Chomsky normal form for context-free grammars. This will help us obtain a fast algorithm for computing stringsums.

\begin{defin}
A bottom-up WPDA is in \defn{normal form} if all of its scanning transitions are $k$-pop, $1$-push for $k \le 2$, and all of its non-scanning transitions are $2$-pop, $1$-push.
Similarly, a top-down WPDA is in \defn{normal form} if all of its scanning transitions are $1$-pop, $k$-push for $k \le 2$, and all of its non-scanning transitions are $1$-pop, $2$-push.
\end{defin}

\subsection{Binarization}

% In this section we give a semantics preserving transformation that will help us obtain a fast algorithm for computing stringsums in a WPDA. We start with the following definition.

Recall that top-down and bottom-up WPDAs have $1$-pop, $k$-push transitions and $k$-pop, $1$-push transitions, respectively. Since the runtime of our stringsum algorithm depends highly on $k$, we convert the WPDA into an equivalent one with $k \leq 2$. 
We call this procedure \defn{binarization} because it is entirely analogous to binarization in CFGs. It is symmetric for top-down and bottom-up WPDAs.

\begin{proposition} \label{thm:bottomup_nf}
Every bottom-up WPDA is equivalent to a bottom-up WPDA whose transitions are $k$-pop, $1$-push where $k \le 2$.
\end{proposition}

\begin{proof}
For each $k$-pop, $1$-push transition $\pdaEdge{w}{p}{a}{q}{\stacksymbol{Y_1} \cdots \stacksymbol{Y_k}}{\stacksymbol{X}}$ such that $k > 2$ we introduce $k-2$ new states $r_1, \cdots, r_{k-2}$ and we replace the original transition with the following: %, as shown in \cref{fig:bottom-up-binarization}:
\begin{align*}
    \pdaEdgeAlign{w}{p}{a}{r_1}{\stacksymbol{Y_1} \stacksymbol{Y_2}}{\stacksymbol{Y_2}} \\
    \pdaEdgeAlign{\one}{r_{i-1}}{\varepsilon}{r_i}{\stacksymbol{Y_i} \stacksymbol{Y_{i+1}}}{\stacksymbol{Y_{i+1}}} \quad  i \in \range{2}{k-2} \\
    \pdaEdgeAlign{\one}{r_{k-2}}{\varepsilon}{q}{\stacksymbol{Y_{k-1}} \stacksymbol{Y_k}}{\stacksymbol{X}}. \tag*{\qedhere}
\end{align*}
\end{proof}

If the original WPDA had transitions that pop at most $k$ symbols, the resulting WPDA has $O(k \cdot |\delta| \cdot |\states|)$ states and $O(k \cdot |\delta|)$ transitions.

\begin{proposition} \label{thm:topdown_nf}
Every top-down WPDA is equivalent to a top-down WPDA whose transitions are $1$-pop, $k$-push where $k \le 2$.
\end{proposition}

\subsection{Nullary Removal}

In this section, we discuss the removal of \defn{nullary} transitions from WPDAs:
\begin{defin}
In a bottom-up WPDA, a transition is called \defn{nullary} if it is of the form $\pdaEdge{w}{p}{\varepsilon}{q}{\varepsilon}{\stacksymbol{X}}$.
\end{defin}
Although nullary transitions are analogous to nullary productions in a CFG, the standard procedure for removing nullary productions from CFGs does not have an exact analogue for PDAs, and the procedure we describe here is novel.

We assume a bottom-up WPDA, but an identical construction exists for top-down WPDAs.
We also assume that the WPDA has been binarized, and semiring $\semiring$ is commutative and continuous.

The construction consists of three steps: partitioning, precomputation, and removal.

\paragraph{Partitioning.}
For every symbol $X \in \stackalphabet$, we replace $X$ with two stack symbols $X^\yesnull$ and $X^\nonull$.
A push computation that pushes a $X^\yesnull$ scans $\varepsilon$,
and a push computation that pushes a $X^\nonull$ scans a string that is not $\varepsilon$. 
To do this, we replace every $k$-pop transition $\pdaEdge{w}{p}{a}{q}{X_1 \cdots X_k}{Y}$ with $2^k$ new transitions
$\pdaEdge{w}{p}{a}{q}{X_1^{\anull_1} \cdots X_k^{\anull_k}}{Y^\anull}$, where $\anull = \yesnull$ iff $\anull_i = \yesnull$ for all $i$ and $a=\varepsilon$.
For instance, we replace transition $\pdaEdge{w}{p}{\varepsilon}{q}{XY}{Z}$ with
the following $2^2 = 4$ transitions
\begin{gather*}
\pdaEdge{w}{p}{\varepsilon}{q}{X^\yesnull Y^\yesnull}{Z^\yesnull} \qquad
\pdaEdge{w}{p}{\varepsilon}{q}{X^\nonull Y^\yesnull}{Z^\nonull} \\
\pdaEdge{w}{p}{\varepsilon}{q}{X^\yesnull Y^\nonull}{Z^\nonull} \qquad
\pdaEdge{w}{p}{\varepsilon}{q}{X^\nonull Y^\nonull}{Z^\nonull}.
\end{gather*}

\paragraph{Precomputation.}
We compute the weight of all non-scanning push computations by solving the quadratic system of equations:
\begin{align*}
\nulltable_{pXq} &= \transweight{p}{\varepsilon}{q}{\varepsilon}{X} \\ &\quad\oplus \bigoplus_{\mathclap{Y,r}} \nulltable_{pYr} \otimes \transweight{r}{\varepsilon}{q}{Y}{X} \\ &\quad\oplus \bigoplus_{\mathclap{Y,Z,s}} \nulltable_{pYZs}
\otimes \transweight{s}{\varepsilon}{q}{YZ}{X} \\
\nulltable_{pYZs} &= \bigoplus_r \nulltable_{pYr} \otimes \nulltable_{rZs}.
\end{align*}
See \cref{sec:runsums} for details on solving such systems of equations, which assumes that $\semiring$ is continuous.
Then $\nulltable_{pXq}$ is the total weight of all push computations of type $\myitem{i}{p}{X}{q}{i}$ for any $i$.
% I chose the "combine XY first" optimization because it's also notationally more convenient elsewhere.

\paragraph{Removal.}

First, delete every transition that pushes $\stacksymbol{X}^\yesnull$ for each $\stacksymbol{X} \in \stackalphabet$. If the PDA accepts~$\varepsilon$ with weight $w$,
% \brian{How do we compute this? \response{david} I'd be just as happy assuming that the PDA rejects $\varepsilon$, as is commonly done for CKY.} 
add $\pdaEdge{w}{\startstate}{\varepsilon}{\finalstate}{\varepsilon}{S^\yesnull}$ as the sole nullary transition. (For correctness, we must also ensure that no transition pops $\stacksymbol{S}^\yesnull$, no transition enters $\startstate$, and no transition leaves $\finalstate$.)

Sometimes an $\stacksymbol{X}^\yesnull$ is popped immediately after it is pushed (that is, with no input symbols scanned between the push and the pop). To handle these cases, for the following transitions, we create new versions in which popped $\stacksymbol{X}^\yesnull$ symbols are removed, and their corresponding weight multiplied in.
\begin{equation*}
\begin{array}{@{}ll@{}}
\text{For each:} & \text{Replace with ($\forall t \in \states$):} \\
\pdaEdge{w}{p}{a}{q}{\stacksymbol{Y}^\yesnull}{\stacksymbol{X}^\nonull} & \pdaEdge{ \nulltable_{t\stacksymbol{Y}p} \otimes w}{t}{a}{q}{\varepsilon}{X^\nonull} \\
\pdaEdge{w}{p}{a}{q}{\stacksymbol{Y}^\yesnull \stacksymbol{Z}^\yesnull}{\stacksymbol{X}^\nonull} & \pdaEdge{\nulltable_{t\stacksymbol{YZ}p} \otimes  w}{t}{a}{q}{\varepsilon}{\stacksymbol{X}^\nonull} \\
\pdaEdge{w}{p}{a}{q}{\stacksymbol{Y}^\nonull \stacksymbol{Z}^\yesnull}{\stacksymbol{X}^\nonull} & \pdaEdge{\nulltable_{t\stacksymbol{Z}p} \otimes w}{t}{a}{q}{\stacksymbol{Y}^\nonull}{\stacksymbol{X}^\nonull}
\end{array}
\end{equation*}
(Note that $a \in \alphabet \cup \{\varepsilon\}$, but the partitioning step only allows $a = \varepsilon$ for the third type above.)

\newcommand{\manynull}[2]{\prescript{}{#1}N_{#2}}
\newcommand{\manynullfused}[3]{\prescript{}{#1#2}#3}

However, we have missed one type of transition, those of the form $\pdaEdge{w}{p}{a}{q}{\stacksymbol{Y}^\yesnull \stacksymbol{Z}^\nonull}{\stacksymbol{X}^\nonull}$.
Create new stack symbols $\manynullfused{r}{s}{{\stacksymbol{Z}}}$ for all $r, s \in \states$ and $\stacksymbol{Z} \in \stackalphabet$. This stands for a sequence of zero or more non-scanning push computations that goes from state $r$ to~$s$, followed by a push computation that pushes~$Z$.
The transition that pushes $Z$ must be a 0-pop transition, because all other transitions expect a symbol of the form $\stacksymbol{X}^\nonull$ on the top of the stack. So we modify (again) the 0-pop transitions to first simulate zero or more nullary transitions:
\begin{equation*}
\begin{array}{@{}ll@{}}
\text{For each:} & \text{Replace with ($\forall s \in \states$):} \\
%\pdaEdge{w}{q}{a}{r}{X^\yesnull}{Y^\nonull} 
\pdaEdge{\nulltable_{tYp} \otimes w}{t}{a}{q}{\varepsilon}{X^\nonull} & \pdaEdge{\nulltable_{tYp} \otimes w}{s}{a}{q}{\varepsilon}{\manynullfused{s}{t}{X}} \\
%\pdaEdge{w}{q}{a}{r}{X^\yesnull Y^\yesnull}{Z^\nonull} 
\pdaEdge{\nulltable_{tYZp} \otimes w}{t}{a}{q}{\varepsilon}{X^\nonull} 
& \pdaEdge{\nulltable_{tYZp} \otimes w}{s}{a}{q}{\varepsilon}{\manynullfused{s}{t}{X}}
\end{array}
\end{equation*}

And for each transition of the form $\pdaEdge{w}{p}{a}{q}{Y^\yesnull Z^\nonull}{X^\nonull}$ (where $a \in \alphabet \cup \{\varepsilon\}$), we create transitions for all $r,s,t \in \states$:
\begin{align*}
\pdaEdgeAlign{\nulltable_{sYt} \otimes w}{p}{a}{q}{\manynullfused{r}{t}{Z}}{\manynullfused{r}{s}{X}}. \\
\intertext{(This step is where commutativity is needed.) Finally, add transitions to remove the state annotations, for all $p, X, q$:}
\pdaEdgeAlign{\one}{q}{\varepsilon}{q}{\manynullfused{p}{p}{X}}{X^\nonull}. 
\end{align*}

%At this point, no stack symbols $X^\yesnull$ remain, so the $\nonull$ and $\yesnull$ annotations can be dropped.

\subsection{Unary Removal}

The final step in conversion to normal form is removal of \defn{unary} transitions, so called by analogy with unary productions in a CFG\@. 
\begin{defin}
A transition is called \defn{unary} if it is of the form $\pdaEdge{w}{p}{\varepsilon}{q}{\stacksymbol{Y}}{\stacksymbol{X}}$.
\end{defin}

We assume that $\semiring$ is equipped with a star operation satisfying $a^* = \one \oplus a \otimes a^* = \one \oplus a^* \otimes a$. If $\semiring$ is continuous, then $a^* = \bigoplus_{i=0}^\infty a^i$.

Unary transitions can form cycles that can be traversed an unbounded number of times, which is problematic for a dynamic programming algorithm. Therefore, we precompute the weights of all runs of unary transitions. 
Define the matrix $\unarytable \in \semiring^{(\states \times \stackalphabet) \times (\states \times \stackalphabet)}$:
\begin{align*}
\unarytable_{p\stacksymbol{Y},q\stacksymbol{X}} = \transweight{p}{\varepsilon}{q}{\stacksymbol{Y}}{\stacksymbol{X}}
\end{align*}
and form its transitive closure $\unarytable^*$ \citep{Lehmann1977}. Then $\unarytable^*_{pY,qX}$ is the total weight of all runs of unary transitions from configuration $\pdaconfig{Y}{p}$ to $\pdaconfig{X}{q}$.

Then we remove all unary transitions and modify every non-unary transition as follows:
\begin{equation*}
\begin{array}{@{}ll@{}}
\text{For each non-unary:} & \text{Replace with:} \\
\pdaEdge{w}{p}{a}{q}{\stackseq}{X} & \pdaEdge{w \otimes \unarytable^*_{qX,rY}}{p}{a}{r}{\stackseq}{Y}
\end{array}
\end{equation*}

We give details on the complexity of this transformation in \cref{sec:appendix-unary-removal}.

\section{Stringsums in Bottom-up WPDAs}

In this section, we give dynamic programming algorithms for computing the stringsum of an input string $\str$ (with $|\str| = n$) of bottom-up WPDAs in normal form. 
We give a basic version of the algorithm, which has the same runtime as Lang's algorithm but improved space requirements, and a fast version that has the same space complexity and runs asymptotically faster. On simple PDAs (for which Lang's algorithm was designed), the latter version has both improved space and time complexity.

\subsection{Basic Algorithm}

\begin{figure} \small
\vspace*{-2ex}
\begin{equation*}
\begin{array}{@{}cc@{}}
\multicolumn{2}{@{}l}{\text{Item form}} \\
\myitem{i}{p}{\stacksymbol{X}}{j}{q} & \begin{gathered} 0 \leq i < j \leq n \\ p, q \in \states; \stacksymbol{X} \in \stackalphabet \end{gathered} \\
\multicolumn{2}{@{}l}{\text{Inference rules}} \\
\dfrac{\myitem{i}{p}{\stacksymbol{Y}}{k}{r} \quad \myitem{k}{r}{\stacksymbol{Z}}{j\compactminus|a|}{s}}{\myitem{i}{p}{\stacksymbol{X}}{j}{q}} & \begin{gathered} \pdaEdge{w}{s}{a}{q}{\stacksymbol{Y} \stacksymbol{Z}}{\stacksymbol{X}} \\ \str\slice{j\compactminus|a|}{j} = a \end{gathered} \\[3ex]
\dfrac{\myitem{i}{p}{\stacksymbol{Y}}{j\compactminus1}{r}}{\myitem{i}{p}{\stacksymbol{X}}{j}{q}} & \pdaEdge{w}{r}{\str_j}{q}{\stacksymbol{Y}}{\stacksymbol{X}} \\
\dfrac{}{\myitem{i}{p}{\stacksymbol{X}}{j}{q}} & \begin{gathered} \pdaEdge{w}{p}{\str_j}{q}{\varepsilon}{\stacksymbol{X}} \\ j = i+1 \end{gathered} \\
\multicolumn{2}{@{}l}{\text{Goal}} \\
\myitem{0}{\startstate}{\stacksymbol{S}}{n}{\finalstate}
\end{array}
\end{equation*}
\caption{Deductive system for stringsums of bottom-up WPDAs in normal form.}
\label{algo:fast-bottom-up-parsing}
\end{figure}

The algorithm computes stringsums efficiently by exploiting the structural similarities among the WPDA runs. \Cref{algo:fast-bottom-up-parsing} shows a deductive system \cite{shieber-1995-principles, goodman-1999-semiring} for deriving items corresponding to push computations.

The items have the form $\myitem{i}{p}{X}{j}{q}$ for $p, q \in \states$, $\stacksymbol{X} \in \stackalphabet$, $0 \leq i \leq j \leq n$. If our algorithm derives this item with weight $w$, then the push computations of type $\myitem{i}{p}{X}{j}{q}$ have total weight $w$.

We distinguish three categories of push computations, based on their final transition, and we include an inference rule for each. First are those consisting of a single $0$-pop, $1$-push transition. The other two categories are those ending in a $1$-pop transition and a $2$-pop transition, respectively. These can be built recursively from shorter push computations. 
%\brian{Why are the transitions written as side conditions? According to Goodman, the top should include anything that contributes to the weight of the bottom.}

The goal item is $\myitem{0}{\startstate}{\stacksymbol{S}}{n}{\finalstate}$, which stands for all runs from the initial configuration to the final configuration that scan $\str$.

\Cref{algo:weighted-push-computations} shows how to compute item weights according to these rules. At termination, the weight of the goal item is the sum of the weights of all accepting runs that scan $\str$.

\begin{algorithm} \small
\caption{Compute the weights of all push computations of a bottom-up WPDA on an input string.}
\label{algo:weighted-push-computations}
\begin{algorithmic}[1] 
%\Func{\weightedpushcomputation{}($\pushdown$, $\str$)}
    \State $\weightstable \gets \zero$
    \State $n \gets |\str|$\label{line:single-push-computation}
    %\For{$p, q \in \states, \stacksymbol{X} \in \stackalphabet, i \in \range{0}{n-1}$}
    \For{$i \in \range{0}{n-1}$}
        \State $j \gets i + 1$
        \LineComment{$0$-pop, $1$-push}
        \For{$(\pdaEdge{w}{p}{\str_j}{q}{\varepsilon}{X}) \in \trans$}
            \State $\itemweight{i}{p}{X}{j}{q} \gets w$
            \label{line:0pop}
        \EndFor
    \EndFor
    \For{$\spanlength \in \range{2}{n}$}
        \For{$i \in \range{0}{n - \spanlength + 1}$}
            \State $j \gets i + \spanlength$
            %\For{$\stacksymbol{X}, \stacksymbol{Y} \in \stackalphabet, p, q, r \in \states$}
            \LineComment{$1$-pop, $1$-push}
            \For{$p \in \states$}
            \label{line:1pop-begin}
                \For{$(\pdaEdge{w}{r}{\str_j}{q}{Y}{X}) \in \trans$}
                \State $\itemweight{i}{p}{X}{j}{q} \opluseq \itemweight{i}{p}{Y}{j\compactminus1}{r} \otimes w$
                    \label{line:1pop-end}
                \EndFor
            \EndFor
            %\For{$\stacksymbol{X}, \stacksymbol{Y}, \stacksymbol{Z} \in \stackalphabet, p, q, r, s \in \states$}
            \LineComment{$2$-pop, $1$-push}
            \For{$p, r \in \states$}
            \label{line:2pop-begin}
                \For{$(\pdaEdge{w}{s}{a}{q}{\stacksymbol{Y}\stacksymbol{Z}}{\stacksymbol{X}}) \in \trans$ with $\str\slice{j\compactminus|a|}{j} = a$} 
                    \For{$k \in \range{i+1}{j-|a|-1}$}
                        \State $\begin{aligned}[t] &\itemweight{i}{p}{\stacksymbol{X}}{j}{q} \opluseq (\itemweight{i}{p}{\stacksymbol{Y}}{k}{r} \otimes {} \\ &\qquad \itemweight{k}{r}{\stacksymbol{Z}}{j-|a|}{s} \otimes w)
                        \end{aligned}
                        $
                    \label{line:2pop-end}
                    \EndFor
                \EndFor
            \EndFor
        \EndFor
    \EndFor
    \State \Return $\itemweight{0}{\startstate}{\stacksymbol{S}}{n}{\finalstate}$
%\EndFunc
\end{algorithmic}

\end{algorithm}

\subsection{Correctness}

\iffalse
\begin{lemma}
\label{lemma:topological-order-items}
For each inference rule of \cref{algo:fast-bottom-up-parsing}, \Cref{algo:weighted-push-computations} 
fully computes the weight of the antecedent items before computing the weight of the consequent item.
\end{lemma}

\begin{proof}
\Cref{algo:weighted-push-computations} computes weights of items $\myitem{i}{p}{X}{j}{q}$ in order of increasing span size ($j-i$), so we just need to show that for each inference rule, the antecedents have smaller spans than the consequent.

For the 2-pop rule, because $i<j$, we have $k-j < k+|a|-i$, and because $j<k$, we have $j-i < k+|a|-i$.

For the 1-pop rule, the definition of normal form guarantees that $|a| > 0$,  so $j+|a|-i > j-i$.
\end{proof}
\fi

\begin{theorem}
Let $\pushdown$ be a WPDA and $\str \in \alphabet^*$ an input string. The weight $\itemweight{i}{p}{X}{j}{q}$ is the total weight of all push computations of $\pushdown$ of type $\myitem{i}{p}{X}{j}{q}$.
\end{theorem}

\begin{proof}
% We prove the statement by induction on the length of the push computations.
By induction on the span length, $\spanlength = j-i$.

\paragraph{Base Case.} 
Assume that $j-i=1$.
The only push computations from state $p$ to $q$ that push $\stacksymbol{X}$ and scan $\str\slice{i}{j}$ are ones that have the single transition
%for which the algorithm computes weights of push computation also of length 1. Let $\arun$ be an arbitrary push computation of length $1$, that is, 
%\begin{equation*}
%    \arun = \pdaconfig{\stackseq_0}{p}, \atrans, \pdaconfig{\stackseq_1}{q}.
%\end{equation*}
%The transition $\atrans$ is of the form
$\atrans = \pdaEdge{w}{p}{\str_j}{q}{\varepsilon}{\stacksymbol{X}}$.
There cannot exist others, because normal form requires that any additional non-scanning transitions would decrease the stack height.
So the total weight of all such push computations is $w$, and the algorithm correctly sets $\itemweight{i}{p}{\stacksymbol{X}}{j}{q} = w$ at line~\ref{line:0pop}.

\paragraph{Inductive Step.} Assume that the statement holds for any
spans of length at most $(\spanlength-1)$ and consider a span of length $\spanlength$. For such spans, the algorithm computes the total weight of all push computations $\arun$ of 
type $\myitem{i}{p}{X}{j}{q}$, for all $\stacksymbol{X} \in \stackalphabet$, $p, q \in \states$, and $j-i = \spanlength$.
%
%the form,
%\begin{equation*}
%    \arun = \pdaconfig{\stackseq_0}{q_0}, \atrans_1, \cdots, \atrans_l, \pdaconfig{\stackseq_l}{q_l},
%\end{equation*}
%where $q_0=p$, $q_l=q$, and $\atrans_l$ is of the form $\pdaEdge{w}{q_{l-1}}{a}{q}{\stackseq}{\stacksymbol{X}}$.
%%where $\stackseq \in \stackalphabet^*, |\stackseq| \in \range{1}{2}$. 
%
This weight must be the sum of weights of three types of push computations: those that end with 0-pop transitions, with 1-pop transitions, and with 2-pop transitions.

But ending with a 0-pop transition is impossible, because such push computations must have only one transition and therefore $j-i \leq 1$.
The 1-pop and 2-pop parts of the sum are computed at lines \ref{line:1pop-begin}--\ref{line:1pop-end} and \ref{line:2pop-begin}--\ref{line:2pop-end} of the algorithm, respectively.

\iffalse
%\cref{fig:stringsum-item-recursive} 
The following equation shows a recursive derivation of an item weight which sums over the push computations that fall into these two categories.

\begin{align*}
    &\itemweight{i}{p}{\stacksymbol{X}}{j}{q} \\   
    %\bigoplus_{\stacksymbol{Y} \in \stackalphabet} \bigoplus_{r \in \states} \bigoplus_{\pdaEdge{w}{q}{\str_j}{r}{\stacksymbol{Y}}{\stacksymbol{X}}} 
    &=\bigoplus_{\mathclap{\substack{\stacksymbol{Y} \in \stackalphabet \\ r \in \states}}} \Bigl(
    \itemweight{i}{p}{\stacksymbol{Y}}{j-1}{q} \otimes
    %w 
    \transweight{q}{\str_j}{r}{Y}{X} \Bigr)
    \\
    %\bigoplus_{\stacksymbol{Y},\stacksymbol{Z} \in \stackalphabet} \bigoplus_{r, s \in \states} \bigoplus_{a \in \set{\str_j, \varepsilon}} \bigoplus_{k \in \range{i+1}{j-|a|-1}} \bigoplus_{\pdaEdge{w}{s}{a}{q}{\stacksymbol{YZ}}{\stacksymbol{X}}} 
    &\oplus \bigoplus_{\mathclap{\substack{k \in \range{i+1}{j-|a|-1} \\ \stacksymbol{Y},\stacksymbol{Z} \in \stackalphabet \\ r, s \in \states \\ a \in \set{\str_j, \varepsilon}}}}
    \; \Bigl(
    \itemweight{i}{p}{\stacksymbol{Y}}{k}{r} \otimes \itemweight{k}{r}{\stacksymbol{Z}}{j-|a|}{s} 
    %\otimes w
    \\[-0.3in]
    &\hspace{5em}\otimes  \transweight{s}{a}{q}{YZ}{X}\Bigr)
\end{align*}
\fi

\subparagraph{Ending with 1-pop transition.} 
The algorithm sums over all possible ending transitions $\lasttrans = \pdaEdge{w}{r}{\str_j}{q}{Y}{X}$. (Normal form requires that this transition be scanning.)
Let $\runs$ be the set of all push computations of type $\myitem{i}{p}{X}{j}{q}$  ending in $\lasttrans$, and let $\runs'$ be the set of all push computations of type $\myitem{i}{p}{Y}{j-1}{r}$. 
Every push computation in $\runs$ must be of the form $\arun = \arun' \circ \lasttrans$, where $\arun' \in \runs'$, and conversely, for every $\arun' \in \runs'$, we have $\arun' \circ \lasttrans \in \runs$.
By the induction hypothesis, the total weight of $\runs'$ was computed in a previous iteration.
Then, by distributivity, we have:
\begin{align*}
    \bigoplus_{\arun \in \runs} \bigotimes_{\atrans \in \arun} \trans(\atrans) &= \bigoplus_{\arun' \in \runs'} \bigotimes_{\atrans \in \arun'} \trans \Parens{\atrans} \otimes \trans \Parens{\lasttrans}  \\
    &= \left( \bigoplus_{\arun' \in \runs'} \bigotimes_{\atrans \in \arun'} \trans \Parens{\atrans} \right) \otimes \trans \Parens{\lasttrans} \\
    &= \itemweight{i}{p}{\stacksymbol{Y}}{j-1}{r} \otimes \trans \Parens{\lasttrans}.
\end{align*}

\subparagraph{Ending with 2-pop transition.} The algorithm sums over all possible ending transitions $\lasttrans = \pdaEdge{w}{s}{a}{q}{\stacksymbol{YZ}}{\stacksymbol{X}}, a \in \set{\str_j, \varepsilon}$.
Every push computation $\arun$ that ends with $\lasttrans$ decomposes uniquely into $\arun' \circ \arun'' \circ \lasttrans$, where $\arun'$ and $\arun''$ are push computations of type $\myitem{i}{p}{Y}{k}{r}$ and $\myitem{k}{r}{Z}{j-|a|}{s}$, respectively, for some $k \in \range{i+1}{j-|a|-1}$ and $r \in \states$. We call $(k,r)$ the \defn{split point} of $\arun$.

The algorithm sums over all split points $(k,r)$.
Let $\runs$ be the set of all push computations of type $\myitem{i}{p}{X}{j}{q}$ ending in $\lasttrans$ with split point $(k,r)$, and let $\runs'$ and $\runs''$ be the sets of all push computations of type $\myitem{i}{p}{Y}{k}{r}$ and $\myitem{k}{r}{Z}{j-|a|}{s}$, respectively.
Every $\arun \in \runs$ must be of the form $\arun' \circ \arun'' \circ \lasttrans$,  where $\arun' \in \runs'$ and $\arun'' \in \runs''$, and conversely, for every $\arun' \in \runs'$ and $\arun'' \in \runs''$, $\arun' \circ \arun'' \circ \lasttrans \in \runs$. Because $i<k$, we must have $j-|a| - k \leq j-k < j-i$, and because $k<j-|a|$, we must have $k-i < j-|a|-i \leq j-i$.
By the induction hypothesis, the total weight of $\runs'$ and $\runs''$ were fully computed in a previous iteration.
As in the previous case, by distributivity we have
\begin{align*}
    &\bigoplus_{\arun \in \runs} \bigotimes_{\atrans \in \arun} \trans \Parens{\atrans} = \itemweight{i}{p}{\stacksymbol{Y}}{k}{r} \\
    &\quad \otimes \itemweight{k}{r}{\stacksymbol{Z}}{j-|a|}{s} \otimes \trans \Parens{\lasttrans}. \tag*{\qedhere}
\end{align*}

\end{proof}

\subsection{Stack Automaton}

The distribution over possible configurations that~$\pushdown$ can be in after reading $\str\slice{0}{m}$ can be generated by a weighted finite-state automaton $\automaton$. The states of $\automaton$ are of the form $(i,q)$, with start state $(0,s)$ and accept states $(m,q)$ for all $q \in \states$. There is a transition $(i,q) \xrightarrow{\stacksymbol{X}/w} (j,r)$ for every item $\myitem{i}{q}{X}{j}{r}$ with weight $w$. Then if an accepting run of $\automaton$ scans $\stackseq$ and ends in state $(m,q)$ with weight $w$, then $\pushdown$ can be in configuration $(q,\stackseq)$ with weight~$w$.

\subsection{Complexity Analysis and Speedup}
\label{sec:analysis_speedup}

For comparison with our algorithm, we show the original algorithm of \citet{lang-1974-deterministic} in \cref{fig:lang}. It has items of the form $\myitem{i}{q}{XY}{j}{r}$, which stands for push computations that start with $X$ as the top stack symbol and push a $Y$ on top of it. 

Our algorithm stores a weight for each item $\myitem{i}{p}{X}{j}{q}$,
% As there is an item for each pair of states $p, q \in \states$, symbol $X \in \stackalphabet$ and pair of indices $i,j \in \range{0}{n}$, the algorithm has 
giving a space complexity of $\mathcal{O} \left(n^2 |\states|^2 |\stackalphabet|. \right)$ % where $n$ is the length of the input string $\str$, $|\states|$ is the number of states and $|\stackalphabet|$ is the cardinality of the stack alphabet. 
This is more efficient than Lang's algorithm, which requires $\mathcal{O} \left(n^2  |\states|^2  |\stackalphabet|^2 \right)$ space. 

\begin{figure} \small
\let\epsilon\varepsilon
\newcommand{\langitem}[6]{\myitem{#6}{#4}{#5#2}{#3}{#1}}
\vspace*{-2ex}
\begin{equation*}
\begin{array}{@{}cc@{}}
\multicolumn{2}{@{}l}{\text{Item form}} \\
\langitem{q}{Y}{j}{p}{X}{i} & \begin{gathered} 0 \leq i < j \leq n \\ p, q \in \states; \stacksymbol{X}, \stacksymbol{Y} \in \stackalphabet \end{gathered} \\
\multicolumn{2}{@{}l}{\text{Inference rules}} \\
\dfrac{}{\langitem{\startstate}{\$}{0}{\startstate}{\$}{0}} \\[2ex]
\dfrac{\langitem{r}{Y}{j\compactminus|a|}{p}{X}{i}}{\langitem{q}{Y}{j}{p}{X}{i}} & \begin{gathered} \pdaEdge{w}{r}{a}{q}{\epsilon}{\epsilon} \\ \str\slice{j\compactminus|a|}{j} = a \end{gathered} \\
\dfrac{\langitem{p}{X}{i}{r}{Z}{k}} {\langitem{q}{Y}{j}{p}{X}{i}} & \begin{gathered} \pdaEdge{w}{p}{a}{q}{\epsilon}{Y} \\ \str\slice{i}{j} = a %, j=i+|a| 
\end{gathered} \\
\dfrac{\langitem{r}{Y}{k}{p}{X}{i}\quad\langitem{s}{Z}{j\compactminus|a|}{r}{Y}{k}}{\langitem{q}{Y}{j}{p}{X}{i}} & \begin{gathered} \pdaEdge{w}{s}{a}{q}{Z}{\epsilon} \\ \str\slice{j\compactminus|a|}{j} = a \end{gathered} \\
\dfrac{\langitem{r}{Z}{j\compactminus|a|}{p}{X}{i}}{\langitem{q}{Y}{j}{p}{X}{i}} & \begin{gathered} \pdaEdge{w}{r}{a}{q}{Z}{Y} \\ \str\slice{j\compactminus|a|}{j} = a \end{gathered} \\
\multicolumn{2}{@{}l}{\text{Goal}} \\
\langitem{\finalstate}{\$}{n}{\startstate}{\$}{0}
%\text{Goal} \qquad \langitem{\finalstate}{\$}{n}{\startstate}{\$}{0}
\end{array}
\end{equation*}
\caption{Deductive system for Lang's algorithm.}
\label{fig:lang}
\end{figure}

Computing the weight of each new item requires, in the worst case (the inference rule for $2$-pop transitions), iterating over stack symbols $Y, Z \in \stackalphabet$, indices $j \in \range{0}{n}$ and states $q,r \in \states$, resulting in a runtime of $\mathcal{O} (n  |\states|^2  |\stackalphabet|^2)$ per item. 
%As there are at most $\mathcal{O}(n^2  |\states|^2  |\stackalphabet|)$ item weights to be derived, 
So the algorithm has a runtime of $\mathcal{O}(n^3  |\states|^4  |\stackalphabet|^3)$, the same as Lang's algorithm.

This runtime can be improved by splitting the inference rule for $2$-pop transitions into two rules:\footnote{The $\backslash$ operator, which is just a punctuation mark and does not require any particular interpretation, was chosen to evoke the $\backslash$ in categorial grammar (using Lambek's ``result on top'' convention): $Y\backslash X$ is an $X$ missing a $Y$ on the left.}
\begin{equation*} \small
\renewcommand{\arraystretch}{2.5}
\begin{array}{@{}cc@{}}
\dfrac{\myitem{k}{r}{Z}{j\compactminus|a|}{s}}{\buhookitem{k}{r}{X}{Y}{j}{q}} &\begin{gathered} \pdaEdge{w}{s}{a}{q}{YZ}{X} \\ \str\slice{j\compactminus|a|}{j} = a \end{gathered} \\
\dfrac{\myitem{i}{p}{Y}{k}{r} \quad \buhookitem{k}{r}{X}{Y}{j}{q}}{\myitem{i}{p}{X}{j}{q}}
\end{array}
\end{equation*}

The first rule has $\mathcal{O}(n^2 |\states|^3 |\stackalphabet|^3)$ instantiations and the second rule has $\mathcal{O}(n^3 |\states|^3 |\stackalphabet|^2)$. So, although we have lost the space-efficiency gain, the total time complexity is now in $\mathcal{O}((n^3 |\stackalphabet|^2 + n^2 |\stackalphabet|^3) |\states|^3)$, a speedup of a factor of more than $|\states|$. We show in \cref{sec:appendix-fast-bottom-up-algorithm} an alternative deductive system that achieves a similar speedup.

Furthermore, Lang's algorithm only works on \simple{} PDAs. To make the algorithms directly comparable, we can assume in the $2$-pop, $1$-push case that $X=Y$. This reduces the space complexity by a factor of $|\stackalphabet|$ again. Moreover, it reduces the number of instantiations of the inference rules above to $\mathcal{O}(n^2 |\states|^3 |\stackalphabet|^2)$ and $\mathcal{O}(n^3 |\states|^3 |\stackalphabet|)$, respectively. So the total time complexity is in $\mathcal{O}(n^3 |\states|^3 |\stackalphabet|^2)$, which is a speedup over Lang's algorithm by a factor of $|\states| \cdot |\stackalphabet|$.

\section{Stringsums in Top-down WPDAs}
\label{sec:algo-top-down}

%The algorithm in the previous section is closely related to Lang's algorithm, and it only works on bottom-up PDAs. But the original definition of PDAs, and still probably the standard definition, corresponds to what we have called top-down PDAs.
Stringsums of weighted top-down WPDAs can be computed by the left/right mirror image of our bottom-up algorithm.
Instead of finding push computations, this algorithm finds
\defn{pop computations}, which decrease (rather than increase) the stack size by exactly one.
\begin{defin}[Pop computation]
Let $\pushdown$ be a top-down WPDA and $\str \in \alphabet^*$ an input string. 
A \defn{pop computation} of type $\myitem{i}{p}{X}{j}{q}$, where $0 \le i \le j \le |\str|$, $p, q \in \states$, and $X \in \stackalphabet$, is a run $\arun = \pdaconfig{\stackseq_0}{q_0}, \ldots, \pdaconfig{\stackseq_m}{q_m}$ that scans $\str\slice{i}{j}$, where $q_0 = p$, $q_m = q$, $\stackseq_0 = \stackseq_m X$, and for all $l < m$, $|\stackseq_l| \geq |\stackseq_0|$.
\end{defin}

% \begin{figure*}
%     \begin{align*}
%         &\myitem{i}{q}{\varepsilon}{i}{q}
%         & {\color{gray} \text{(Axioms)}} \\
%         &\frac{\myitem{i_1}{q_1}{Y_1}{i_2}{q_2} \quad \myitem{i_2}{q_2}{Y_2}{i_3}{q_3} \quad \cdots \quad \myitem{i_k}{q_k}{Y_k}{i_{k+1}}{q_{k+1}}}{\myitem{i_0-|a|}{q_0}{X}{i_{k+1}}{q_{k+1}}} \side{\pdaEdge{w}{q_0}{a}{q_1}{X}{Y_1 Y_2 \cdots Y_k}}
%             & {\color{gray} \text{\quad \quad (Inference rule)}} \\
%         &\myitem{0}{s}{S}{n}{f} & {\color{gray} \text{\quad \quad (Goal)}}
%     \end{align*}
% \caption{Deductive system for computing stringsums in top-down WPDAs.}
% \end{figure*}

\begin{figure} \small
\vspace*{-2ex}
\begin{equation*}
\begin{array}{@{}cc@{}}
\multicolumn{2}{@{}l}{\text{Item form}} \\
\myitem{i}{p}{\stacksymbol{X}}{j}{q} & \begin{gathered} 0 \leq i < j \leq n \\ p, q \in \states; \stacksymbol{X} \in \stackalphabet \end{gathered} \\
\multicolumn{2}{@{}l}{\text{Inference rules}} \\
\dfrac{\myitem{i\compactplus|a|}{r}{\stacksymbol{Y}}{k}{s} \quad \myitem{k}{s}{\stacksymbol{Z}}{j}{q}} {\myitem{i}{p}{\stacksymbol{X}}{j}{q}} &\begin{gathered} \pdaEdge{w}{p}{a}{r}{\stacksymbol{X}}{\stacksymbol{Y} \stacksymbol{Z}} \\ \str\slice{i}{i\compactplus|a|} = a \end{gathered} \\[3ex]
\dfrac{\myitem{i\compactplus1}{r}{\stacksymbol{Y}}{j}{q}}{\myitem{i}{p}{\stacksymbol{X}}{j}{q}} &\pdaEdge{w}{p}{\str_{i\compactplus1}}{r}{\stacksymbol{X}}{\stacksymbol{Y}} \\ \dfrac{}{\myitem{i}{p}{\stacksymbol{X}}{j}{q}} &\begin{gathered} \pdaEdge{w}{p}{\str_{i\compactplus1}}{q}{\stacksymbol{X}}{\varepsilon} \\ 
i = j-1 \end{gathered} \\
\multicolumn{2}{@{}l}{\text{Goal}} \\
\myitem{0}{\startstate}{\stacksymbol{S}}{n}{\finalstate}
%\text{Goal} \qquad \myitem{0}{\startstate}{\stacksymbol{S}}{n}{\finalstate}
\end{array}
\end{equation*}
\caption{Deductive system for stringsums of top-down WPDAs in normal form.}
\label{algo:fast-top-down-parsing}
\end{figure}

\Cref{algo:fast-top-down-parsing} shows the inference rules used by the dynamic program, which are the mirror image of the rules in \cref{algo:fast-bottom-up-parsing}. Each item $\myitem{i}{p}{X}{j}{q}$, which stands for a set of pop computations, is derived using a transition and items corresponding to pop computations that happen \emph{later} in the run.

\subsection{Comparison with Lang's algorithm}

Since top-down PDAs are more standard, and the only direct PDA stringsum algorithm in the literature is Lang's algorithm, it might have seemed natural to extend Lang's algorithm to top-down PDAs, as is done by \citet{dusell-chiang-2020-learning}. Like Lang's algorithm, their algorithm has items of the form $\myitem{i}{q}{XY}{j}{r}$, but unlike Lang's algorithm, it requires the $X$ in order to support 1-pop, 2-push transitions.
As a result, their algorithm has space complexity $\mathcal{O}(n^2 |\states|^2 |\stackalphabet|^2)$ and time complexity $\mathcal{O}(n^3 |\states|^4 |\stackalphabet|^3)$, like Lang's algorithm.
But if they had used our algorithm for top-down WPDAs, using pop computations, they would have saved a factor of $|\stackalphabet|$ space, and because their 1-pop, 2-push transitions never change the popped symbol, they would have also saved a factor of $|\states| \cdot |\stackalphabet|$ time.

\subsection{Experiment}

To give a concrete example, we consider the renormalizing nondeterministic stack RNN (RNS-RNN) of \citet{dusell-chiang-2022-learning}, which uses Lang's algorithm (\cref{fig:lang}) on a top-down PDA\@. Since the RNN must process the string from left to right, we cannot use the bottom-up stringsum algorithm, but we can still apply the speedup of \cref{sec:analysis_speedup}, splitting the 1-pop, 0-push rule of \cref{fig:lang} into two rules.
\iffalse
\begin{equation*} \small
\let\epsilon\varepsilon
\newcommand{\langitem}[6]{\myitem{#6}{#4}{#5#2}{#3}{#1}}
\newcommand{\tempitem}[5]{\langle #1, #2, #3, #4, #5 \rangle}
\renewcommand{\arraystretch}{2.5}
\begin{array}{@{}cc@{}}
\dfrac{\langitem{s}{Z}{j\compactminus|a|}{r}{Y}{k}}{\tempitem{k}{r}{Y}{j}{q}} & \begin{gathered} \pdaEdge{w}{s}{a}{q}{Z}{\epsilon} \\ \str\slice{j\compactminus|a|}{j} = a \end{gathered} \\
\dfrac{\langitem{r}{Y}{k}{p}{X}{i} \tempitem{k}{r}{Y}{j}{q}}{\langitem{q}{Y}{j}{p}{X}{i}} &
\end{array}
\end{equation*}
\fi
Again, this decreases the time complexity from $\mathcal{O}(n^3  |\states|^4  |\stackalphabet|^3)$ to $\mathcal{O}((n^3 |\stackalphabet|^2 + n^2 |\stackalphabet|^3) |\states|^3)$. When we compare the two implementations on a corpus of strings whose lengths were drawn from $[40, 80]$ on a NVIDIA GeForce RTX 2080 Ti GPU, when $|\states| = 5$ and $|\stackalphabet| = 3$, the new version is about 10 times as fast (Figure~\ref{fig:rnsrnn}).

\begin{figure}
\centering
\iffalse
\begin{tabular}{crrrr}
\toprule
& \multicolumn{2}{c}{original} & \multicolumn{2}{c}{speedup} \\
$|\states|$ & time & space & time & space \\
\midrule
1 & 0.56 & 6 & 0.63 & 3 \\
2 & 0.77 & 80 & 0.63 & 19 \\
3 & 1.87 & 385 & 0.71 & 57 \\
4 & 4.86 & 1195 & 0.85 & 125 \\
5 & 11.26 & 2894 & 1.13 & 233 \\
\bottomrule
\end{tabular}
\fi

\pgfplotsset{legend pos={north west},height=1.5in,width=3in,every axis y label/.append style={at={(0,0.5)},yshift=1.5em}}
\tikzset{-}
\begin{tabular}{r}
\begin{tikzpicture}
\begin{axis}[ylabel={time (s)}]
\addplot[myblue,mark=*] coordinates { (1,0.56) (2,0.77) (3,1.87) (4,4.86) (5,11.26) };
\addlegendentry {original}
\addplot[myred,mark=square*] coordinates { (1,0.63) (2,0.63) (3,0.71) (4,0.85) (5,1.13) };
\addlegendentry {speedup}
\end{axis}
\end{tikzpicture}
\\
\begin{tikzpicture}
\begin{axis}[xlabel={states $|\states|$},ylabel={space (GB)}]
\addplot[myblue,mark=*] coordinates { (1,.006) (2,.080) (3,.385) (4,1.195) (5,2.894) };
\addlegendentry {original}
\addplot[myred,mark=square*] coordinates { (1,.003) (2,.019) (3,.057) (4,.125) (5,.233) };
\addlegendentry {speedup}
\end{axis}
\end{tikzpicture}
\end{tabular}
\caption{Applying our speedup to the RNS-RNN, which uses Lang's algorithm adapted to top-down PDAs, yields dramatic time and space savings.}
\label{fig:rnsrnn}
\end{figure}

\subsection{Comparison with CFG/CKY}

We also compare our stringsum algorithm with converting a top-down PDA to a CFG and computing stringsums using CKY\@.
The usual conversion from top-down normal form PDAs to CFGs \citep{hopcroft-2006-introduction} creates a CFG with $O(|\states|^2 |\stackalphabet|)$ nonterminal symbols, so CKY would take $O(n^3 |\states|^6 |\stackalphabet|^3)$ time. Our algorithm thus represents a speedup of more than $|\states|^3 $. Of course, various optimizations could be made to improve this time, and in particular there is an optimization \citep{eisner-blatz-2007} analogous to the speedup in \cref{sec:analysis_speedup}.

%For bottom-up PDAs, we are not aware of a standard conversion to CFG, so our algorithms are more general.

\section{\Runsum{}s in Bottom-up WPDAs}
\label{sec:runsums}

We can use a notion of push computation similar to \cref{def:push-computation} to derive a space-efficient algorithm for computing \runsum{}s in bottom-up WPDAs. The item $\runsumitem{p}{\stacksymbol{X}}{q}$ stands for runs from state $p$ to state $q$ that push the symbol $\stacksymbol{X}$ on top of the stack while leaving the symbols underneath unchanged. 

\begin{defin}[Push computation]
Let $\pushdown$ be a bottom-up WPDA\@. A \defn{push computation} of type $\runsumitem{p}{X}{q}$, where $p, q \in \states$, and $X \in \stackalphabet$, is a run $\arun = \pdaconfig{\stackseq_0}{q_0}, \ldots, \pdaconfig{\stackseq_n}{q_n}$, where $q_0 = p$, $q_n = q$, $\stackseq_n = \stackseq_0 \stacksymbol{X}$, and for all $i>0$, $|\stackseq_i| \geq |\stackseq_n|$. 
\end{defin}

These items closely resemble those used for computing stringsums, but discard the two variables $i, j$ that we used for indexing substrings of the input, as we are interested in computing the total weight of runs that scan \emph{any} string.

\begin{defin}
Let $\runs \Parens{p, \stacksymbol{X}, q}$ be the set containing all push computations from state $p$ to state $q$ that push $\stacksymbol{X}$. The \runsum{} $\runsumitemweight{p}{X}{q}$ is defined as 
\begin{equation*}
    \runsumitemweight{p}{\stacksymbol{X}}{q} = \bigoplus_{\arun \in \runs \Parens{p, \stacksymbol{X}, q}} \normalizer (\arun).
\end{equation*}
\end{defin}

% \begin{figure*}
% \begin{align*}
%     Z({\runsumitem{q_0}{X}{q_{k+1}}}_1) &= \bigoplus_{q_1 \cdots q_k} \bigoplus_{Y_1 \cdots Y_k} \bigotimes_{i=1}^k Z({\runsumitem{q_{i-1}}{Y_i}{q_i}}_1) \otimes {\weight{\pdaEdge{w}{q_k}{a}{q_{k+1}}{Y_1 \cdots Y_k}{X}}}_1 \\
%     Z({\runsumitem{q_0}{X}{q_{k+1}}}_2) &= \bigoplus_{q_1 \cdots q_k} \bigoplus_{Y_1 \cdots Y_k} \bigotimes_{i=1}^k Z({\runsumitem{q_{i-1}}{Y_i}{q_i}}_2) \otimes {\weight{\pdaEdge{w}{q_k}{a}{q_{k+1}}{Y_1 \cdots Y_k}{X}}}_2 \\
%     &\vdots \\
%     Z({\runsumitem{q_0}{X}{q_{k+1}}}_i) &= \bigoplus_{q_1 \cdots q_k} \bigoplus_{Y_1 \cdots Y_k} \bigotimes_{i=1}^k Z({\runsumitem{q_{i-1}}{Y_i}{q_i}}_i) \otimes {\weight{\pdaEdge{w}{q_k}{a}{q_{k+1}}{Y_1 \cdots Y_k}{X}}}_i \\
%     &\vdots \\
%     Z({\runsumitem{q_0}{X}{q_{k+1}}}_n) &= \bigoplus_{q_1 \cdots q_k} \bigoplus_{Y_1 \cdots Y_k} \bigotimes_{i=1}^k Z({\runsumitem{q_{i-1}}{Y_i}{q_i}}_n) \otimes {\weight{\pdaEdge{w}{q_k}{a}{q_{k+1}}{Y_1 \cdots Y_k}{X}}}_n
% \end{align*}
% \caption{\Runsum{} of a WPDA expressed as a system of non-linear equations.}
% \label{fig:runsum-system-equations}
% \end{figure*}

The \runsum{} of a set of push computations can be expressed in terms of other \runsum{}s:
\begin{align*}
    &\runsumitemweight{p}{X}{q} = \bigoplus_{\mathclap{a \in \alphabet \cup \set{\varepsilon}}} %\weight{\pdaEdge{w}{p}{a}{q}{\varepsilon}{\stacksymbol{X}}}
    \transweight{p}{a}{q}{\varepsilon}{X}
    \\
    %\bigoplus_{\stacksymbol{Y} \in \stackalphabet} \bigoplus_{r \in \states} \bigoplus_{a \in \alphabet \cup \set{\varepsilon}} 
    &\oplus \bigoplus_{\mathclap{\substack{\stacksymbol{Y} \in \stackalphabet \\ r \in \states \\ a \in \alphabet \cup \set{\varepsilon}}}}
    \runsumitemweight{p}{\stacksymbol{Y}}{r} \otimes %\weight{\pdaEdge{w}{r}{a}{q}{\stacksymbol{Y}}{\stacksymbol{X}}} 
    \transweight{r}{a}{q}{Y}{X}
    \\
    %\bigoplus_{\stacksymbol{Y}, \stacksymbol{Z} \in \stackalphabet} \bigoplus_{r, s \in \states} \bigoplus_{a \in \alphabet \cup \set{\varepsilon}} 
    &\oplus\bigoplus_{\mathclap{\substack{\stacksymbol{Y}, \stacksymbol{Z} \in \stackalphabet \\ r, s \in \states \\ a \in \alphabet \cup \set{\varepsilon}}}}
    \runsumitemweight{p}{\stacksymbol{Y}}{r} \otimes \runsumitemweight{r}{\stacksymbol{Z}}{s} \otimes
    %\weight{\pdaEdge{w}{s}{a}{q}{\stacksymbol{YZ}}{\stacksymbol{X}}}
    \transweight{s}{a}{q}{YZ}{X}
\end{align*}
In general, \runsum{}s cannot be computed recursively, as $\runsumitemweight{p}{\stacksymbol{X}}{q}$ may rely on \runsum{}s that are yet to be computed. Instead, we assume that $\semiring$ is continuous and solve the system of nonlinear equations using fixed-point iteration or the semiring generalization of Newton's method \citep{esparza-2007-fixed}.

This algorithm computes $\mathcal{O}(|\states|^2  |\stackalphabet|)$ items. An allsum algorithm based on Lang's algorithm would have computed $\mathcal{O}(|\states|^2  |\stackalphabet|^2)$ items; thus we have reduced the space complexity by a factor of $|\stackalphabet|$.

\section{Conclusion}

%In this paper we gave a study of algorithms for inference in weighted pushdown automata. Our definition of WPDAs is more general than previous definitions, and we showed how to convert our WPDAs into the standard (top-down) definition or its (bottom-up) mirror image. We derived novel algorithms for both of these variants and showed that they are more efficient and/or more general than Lang's algorithm, the only similar algorithm that we are aware of.

Our study has contributed several results and algorithms whose weighted CFG analogues have long been known, but have previously been missing for weighted PDAs---a normal form analogous to Chomsky normal form and a stringsum algorithm analogous to weighted CKY\@. But it has also revealed some important differences, confirming that the study of weighted PDAs is of interest in its own right. Most notably, we identified two different normal forms and two corresponding stringsum algorithms (and two \runsum{} algorithms). Since the only existing PDA stringsum algorithm we are aware of, Lang's algorithm, is better suited to bottom-up PDAs, whereas the more standard definition of PDAs is of top-down PDAs, our algorithm for top-down WPDAs fills a significant gap.

\section*{Acknowledgements}

This material is based upon work supported by the National Science Foundation under Grant No.~CCF-2019291. Any opinions, findings, and conclusions or recommendations expressed in this material are those of the author(s) and do not necessarily reflect the views of the National Science Foundation.

\section*{Limitations}

Removal of nullary transitions, while similar to removal of nullary rules from a WCFG, is conceptually more complicated, and while it has the same asymptotic complexity, it's currently unknown how the two would compare in practice. Additionally, our nullary removal construction requires a commutative semiring, while removal of nullary productions from a WCFG does not.

Our algorithm for top-down WPDAs processes a string from right to left, which may be undesirable in some NLP applications and in models of human sentence processing.

\section*{Ethics Statement}

The authors foresee no ethical concerns with the research presented in this paper. 

% \bibliography{anthology, custom}
\bibliography{custom}

\begin{thebibliography}{36}
\expandafter\ifx\csname natexlab\endcsname\relax\def\natexlab#1{#1}\fi

\bibitem[{Aho and Ullman(1972)}]{aho+ullman:1972}
Alfred~V. Aho and Jeffrey~D. Ullman. 1972.
\newblock \href {https://dl.acm.org/doi/book/10.5555/578789} {\emph{The Theory
  of Parsing, Translation, and Compiling}}, volume~1.
\newblock Prentice-Hall.

\bibitem[{Allauzen et~al.(2014)Allauzen, Byrne, de~Gispert, Iglesias, and
  Riley}]{allauzen-etal-2014-pushdown}
Cyril Allauzen, Bill Byrne, Adri{\`a} de~Gispert, Gonzalo Iglesias, and Michael
  Riley. 2014.
\newblock \href {https://doi.org/10.1162/COLI_a_00197} {Pushdown automata in
  statistical machine translation}.
\newblock \emph{Computational Linguistics}, 40(3):687--723.

\bibitem[{Andor et~al.(2016)Andor, Alberti, Weiss, Severyn, Presta, Ganchev,
  Petrov, and Collins}]{andor-etal-2016-globally}
Daniel Andor, Chris Alberti, David Weiss, Aliaksei Severyn, Alessandro Presta,
  Kuzman Ganchev, Slav Petrov, and Michael Collins. 2016.
\newblock \href {https://doi.org/10.18653/v1/P16-1231} {Globally normalized
  transition-based neural networks}.
\newblock In \emph{Proceedings of the 54th Annual Meeting of the Association
  for Computational Linguistics (Volume 1: Long Papers)}, pages 2442--2452.

\bibitem[{Autebert et~al.(1997)Autebert, Berstel, and Boasson}]{autebert+:1997}
Jean-Michel Autebert, Jean Berstel, and Luc Boasson. 1997.
\newblock \href {https://doi.org/10.1007/978-3-642-59136-5_3} {Context-free
  languages and pushdown automata}.
\newblock In \emph{Handbook of Formal Languages}, volume~1, pages 111--174.

\bibitem[{Bar-Hillel et~al.(1961)Bar-Hillel, Perles, and
  Shamir}]{bar-hillel-1961-formal}
Y.~Bar-Hillel, M.~Perles, and E.~Shamir. 1961.
\newblock \href {https://www.proquest.com/docview/1299532646} {On formal
  properties of simple phrase structure grammars}.
\newblock \emph{{Z}eitschrift {f\" ur} {P}honetik, {S}prachwissenschaft und
  {K}ommunikationsforschung}, 14:143--172.

\bibitem[{Chen and Manning(2014)}]{chen-manning-2014-fast}
Danqi Chen and Christopher Manning. 2014.
\newblock \href {https://doi.org/10.3115/v1/D14-1082} {A fast and accurate
  dependency parser using neural networks}.
\newblock In \emph{Proceedings of the 2014 Conference on Empirical Methods in
  Natural Language Processing}, pages 740--750.

\bibitem[{Chomsky(1963)}]{chomsky:1963}
Noam Chomsky. 1963.
\newblock \href
  {https://archive.org/details/handbookofmathem017893mbp/page/323/mode/2up}
  {Formal properties of grammars}.
\newblock In \emph{Handbook of Mathematical Psychology}, volume~2, pages
  323--418. John Wiley \& Sons.

\bibitem[{Droste and Kuich(2009)}]{droste-2009-semirings}
Manfred Droste and Werner Kuich. 2009.
\newblock \href {https://doi.org/10.1007/978-3-642-01492-5_1} {Semirings and
  formal power series}.
\newblock In \emph{Handbook of Weighted Automata}, pages 3--28.

\bibitem[{DuSell and Chiang(2020)}]{dusell-chiang-2020-learning}
Brian DuSell and David Chiang. 2020.
\newblock \href {https://doi.org/10.18653/v1/2020.conll-1.41} {Learning
  context-free languages with nondeterministic stack {RNN}s}.
\newblock In \emph{Proceedings of the 24th Conference on Computational Natural
  Language Learning}, pages 507--519.

\bibitem[{DuSell and Chiang(2022)}]{dusell-chiang-2022-learning}
Brian DuSell and David Chiang. 2022.
\newblock \href {https://openreview.net/forum?id=5LXw_QplBiF} {Learning
  hierarchical structures with differentiable nondeterministic stacks}.
\newblock In \emph{International Conference on Learning Representations}.

\bibitem[{Dyer et~al.(2015)Dyer, Ballesteros, Ling, Matthews, and
  Smith}]{dyer-etal-2015-transition}
Chris Dyer, Miguel Ballesteros, Wang Ling, Austin Matthews, and Noah~A. Smith.
  2015.
\newblock \href {https://doi.org/10.3115/v1/P15-1033} {Transition-based
  dependency parsing with stack long short-term memory}.
\newblock In \emph{Proceedings of the 53rd Annual Meeting of the Association
  for Computational Linguistics and the 7th International Joint Conference on
  Natural Language Processing (Volume 1: Long Papers)}, pages 334--343.

\bibitem[{Dyer et~al.(2016)Dyer, Kuncoro, Ballesteros, and
  Smith}]{dyer-etal-2016-recurrent}
Chris Dyer, Adhiguna Kuncoro, Miguel Ballesteros, and Noah~A. Smith. 2016.
\newblock \href {https://doi.org/10.18653/v1/N16-1024} {Recurrent neural
  network grammars}.
\newblock In \emph{Proceedings of the 2016 Conference of the North {A}merican
  Chapter of the Association for Computational Linguistics: Human Language
  Technologies}, pages 199--209.

\bibitem[{Earley(1970)}]{earley-1970-efficient}
Jay Earley. 1970.
\newblock \href {https://doi.org/10.1145/362007.362035} {An efficient
  context-free parsing algorithm}.
\newblock \emph{Communications of the ACM}, 13(2):94--102.

\bibitem[{Eisner and Blatz(2007)}]{eisner-blatz-2007}
Jason Eisner and John Blatz. 2007.
\newblock \href
  {https://web.stanford.edu/group/cslipublications/cslipublications/FG/2006/eisner.pdf}
  {Program transformations for optimization of parsing algorithms and other
  weighted logic programs}.
\newblock In \emph{Proceedings of the 11th Conference on Formal Grammar}, pages
  45--85.

\bibitem[{Esparza et~al.(2007)Esparza, Kiefer, and
  Luttenberger}]{esparza-2007-fixed}
Javier Esparza, Stefan Kiefer, and Michael Luttenberger. 2007.
\newblock \href
  {https://link.springer.com/chapter/10.1007/978-3-540-70918-3_26} {On fixed
  point equations over commutative semirings}.
\newblock In \emph{24th Annual Symposium on Theoretical Aspects of Computer
  Science}, pages 296--307.

\bibitem[{Evey(1963)}]{evey:1963}
R.~James Evey. 1963.
\newblock \href {https://doi.org/10.1145/1463822.1463848} {Application of
  pushdown-store machines}.
\newblock In \emph{{AFIPS} '63: Proceedings of the November 12--14, 1963, Fall
  Joint Computer Conference}, pages 215--–227.

\bibitem[{Fern{\'a}ndez-Gonz{\'a}lez and
  G{\'o}mez-Rodr{\'\i}guez(2019)}]{fernandez-gonzalez-gomez-rodriguez-2019-left}
Daniel Fern{\'a}ndez-Gonz{\'a}lez and Carlos G{\'o}mez-Rodr{\'\i}guez. 2019.
\newblock \href {https://doi.org/10.18653/v1/N19-1076} {Left-to-right
  dependency parsing with pointer networks}.
\newblock In \emph{Proceedings of the 2019 Conference of the North {A}merican
  Chapter of the Association for Computational Linguistics: Human Language
  Technologies, Volume 1 (Long and Short Papers)}, pages 710--716.

\bibitem[{Goodman(1999)}]{goodman-1999-semiring}
Joshua Goodman. 1999.
\newblock \href {https://aclanthology.org/J99-4004} {Semiring parsing}.
\newblock \emph{Computational Linguistics}, 25(4):573--606.

\bibitem[{Hopcroft et~al.(2006)Hopcroft, Motwani, and
  Ullman}]{hopcroft-2006-introduction}
John~E. Hopcroft, Rajeev Motwani, and Jeffrey~D. Ullman. 2006.
\newblock \href {https://dl.acm.org/doi/book/10.5555/1196416}
  {\emph{Introduction to Automata Theory, Languages, and Computation}}, 3rd
  edition.
\newblock Addison-Wesley Longman Publishing Co.

\bibitem[{Huang et~al.(2009)Huang, Jiang, and
  Liu}]{huang-etal-2009-bilingually}
Liang Huang, Wenbin Jiang, and Qun Liu. 2009.
\newblock \href {https://aclanthology.org/D09-1127} {Bilingually-constrained
  (monolingual) shift-reduce parsing}.
\newblock In \emph{Proceedings of the 2009 Conference on Empirical Methods in
  Natural Language Processing}, pages 1222--1231.

\bibitem[{Knuth(1965)}]{knuth-1965-translation}
Donald~E. Knuth. 1965.
\newblock \href {https://doi.org/10.1016/S0019-9958(65)90426-2} {On the
  translation of languages from left to right}.
\newblock \emph{Information and Control}, 8(6):607--639.

\bibitem[{Lang(1974)}]{lang-1974-deterministic}
Bernard Lang. 1974.
\newblock \href
  {https://link.springer.com/chapter/10.1007/978-3-662-21545-6_18}
  {Deterministic techniques for efficient non-deterministic parsers}.
\newblock In \emph{ICALP 1974: Automata, Languages and Programming}, pages
  255--269.

\bibitem[{Lehmann(1977)}]{Lehmann1977}
Daniel~J. Lehmann. 1977.
\newblock \href {https://doi.org/https://doi.org/10.1016/0304-3975(77)90056-1}
  {Algebraic structures for transitive closure}.
\newblock \emph{Theoretical Computer Science}, 4(1):59--76.

\bibitem[{Lewis and Papadimitriou(1997)}]{lewis+papadimitriou:1997}
Harry~R. Lewis and Christos~H. Papadimitriou. 1997.
\newblock \href {https://dl.acm.org/doi/10.5555/549820} {\emph{Elements of the
  Theory of Computation}}, 2nd edition.
\newblock Prentice-Hall.

\bibitem[{Ma et~al.(2018)Ma, Hu, Liu, Peng, Neubig, and
  Hovy}]{ma-etal-2018-stack}
Xuezhe Ma, Zecong Hu, Jingzhou Liu, Nanyun Peng, Graham Neubig, and Eduard
  Hovy. 2018.
\newblock \href {https://doi.org/10.18653/v1/P18-1130} {Stack-pointer networks
  for dependency parsing}.
\newblock In \emph{Proceedings of the 56th Annual Meeting of the Association
  for Computational Linguistics (Volume 1: Long Papers)}, pages 1403--1414.

\bibitem[{Nivre(2003)}]{nivre-2003-efficient}
Joakim Nivre. 2003.
\newblock \href {https://aclanthology.org/W03-3017} {An efficient algorithm for
  projective dependency parsing}.
\newblock In \emph{Proceedings of the Eighth International Conference on
  Parsing Technologies}, pages 149--160.

\bibitem[{Nivre(2004)}]{nivre-2004-incrementality}
Joakim Nivre. 2004.
\newblock \href {https://aclanthology.org/W04-0308} {Incrementality in
  deterministic dependency parsing}.
\newblock In \emph{Proceedings of the Workshop on Incremental Parsing: Bringing
  Engineering and Cognition Together}, pages 50--57.

\bibitem[{Resnik(1992)}]{resnik-1992-left}
Philip Resnik. 1992.
\newblock \href {https://aclanthology.org/C92-1032} {Left-corner parsing and
  psychological plausibility}.
\newblock In \emph{{COLING} 1992 Volume 1: The 14th {I}nternational
  {C}onference on {C}omputational {L}inguistics}, pages 191--197.

\bibitem[{Roark(2001)}]{roark-2001-probabilistic}
Brian Roark. 2001.
\newblock \href {https://doi.org/10.1162/089120101750300526} {Probabilistic
  top-down parsing and language modeling}.
\newblock \emph{Computational Linguistics}, 27(2):249--276.

\bibitem[{Sch{\"u}tzenberger(1963)}]{schutzenberger:1963}
Marcel-Paul Sch{\"u}tzenberger. 1963.
\newblock \href {https://doi.org/https://doi.org/10.1016/S0019-9958(63)90306-1}
  {On context-free languages and push-down automata}.
\newblock \emph{Information and Control}, 6(3):246--264.

\bibitem[{Shi et~al.(2017)Shi, Huang, and Lee}]{shi-etal-2017-fast}
Tianze Shi, Liang Huang, and Lillian Lee. 2017.
\newblock \href {https://doi.org/10.18653/v1/D17-1002} {Fast(er) exact decoding
  and global training for transition-based dependency parsing via a minimal
  feature set}.
\newblock In \emph{Proceedings of the 2017 Conference on Empirical Methods in
  Natural Language Processing}, pages 12--23.

\bibitem[{Shieber et~al.(1995)Shieber, Schabes, and
  Pereira}]{shieber-1995-principles}
Stuart~M. Shieber, Yves Schabes, and Fernando~C.N. Pereira. 1995.
\newblock \href {https://doi.org/10.1016/0743-1066(95)00035-I} {Principles and
  implementation of deductive parsing}.
\newblock \emph{Journal of Logic Programming}, 24:3--36.

\bibitem[{Sipser(2012)}]{sipser:2012}
Michael Sipser. 2012.
\newblock \href
  {https://www.cengage.uk/c/introduction-to-the-theory-of-computation-3e-3e-sipser/9781473778092/}
  {\emph{Introduction to the Theory of Computation}}, 3rd edition.
\newblock Cengage Learning.

\bibitem[{Stolcke(1995)}]{stolcke-1995-efficient}
Andreas Stolcke. 1995.
\newblock \href {https://aclanthology.org/J95-2002} {An efficient probabilistic
  context-free parsing algorithm that computes prefix probabilities}.
\newblock \emph{Computational Linguistics}, 21(2):165--201.

\bibitem[{Weiss et~al.(2015)Weiss, Alberti, Collins, and
  Petrov}]{weiss-etal-2015-structured}
David Weiss, Chris Alberti, Michael Collins, and Slav Petrov. 2015.
\newblock \href {https://doi.org/10.3115/v1/P15-1032} {Structured training for
  neural network transition-based parsing}.
\newblock In \emph{Proceedings of the 53rd Annual Meeting of the Association
  for Computational Linguistics and the 7th International Joint Conference on
  Natural Language Processing (Volume 1: Long Papers)}, pages 323--333.

\bibitem[{Zhang and Clark(2008)}]{zhang-clark-2008-tale}
Yue Zhang and Stephen Clark. 2008.
\newblock \href {https://aclanthology.org/D08-1059} {A tale of two parsers:
  {I}nvestigating and combining graph-based and transition-based dependency
  parsing}.
\newblock In \emph{Proceedings of the 2008 Conference on Empirical Methods in
  Natural Language Processing}, pages 562--571.

\end{thebibliography}
\bibliographystyle{acl_natbib}

\clearpage
\appendix

\section{Details of Unary Removal}
\label{sec:appendix-unary-removal}

Since $\unarytable$ is a $|\states| \times |\stackalphabet|$ matrix, computing its transitive closure takes $\mathcal{O}((|\states||\stackalphabet|)^3)$ time. 
However, if we perform nullary removal first, the stack alphabet could grow by a factor of $|\states|^2$ because of the stack symbols $\manynullfused{r}{s}{Z}$, which would seem to make the transitive closure take $\mathcal{O}((|\states|^3|\stackalphabet|)^3)$ time. 

For comparison, if we converted the PDA to a CFG, it would have $\mathcal{O}(|\states|^2|\stackalphabet|)$ nonterminals, so computing the transitive closure of the unary rules would take $\mathcal{O}((|\states|^2|\stackalphabet|)^3)$ time.

But the graph formed by the unary transitions can be decomposed into several strongly connected components (SCCs), many of which are identical, so the transitive closure can be sped up considerably.
Define three matrices for three different forms of unary transitions:
\begin{align*}
\unarytable^1_{ptZ,qsX} &=  \transweight{p}{\varepsilon}{q}{\manynullfused{r}{t}{Z}}{\manynullfused{r}{s}{X}} \\
\unarytable^2_{qpX,qX} &= 
\transweight{q}{\varepsilon}{q}{\manynullfused{p}{p}{X}}{X^\nonull} \\
\unarytable^3_{pY,qX} &= \transweight{p}{\varepsilon}{q}{Y^\nonull}{X^\nonull}.
\end{align*}
There are no transitions of the form $\pdaEdge{w}{p}{\varepsilon}{q}{Y^\nonull}{\manynullfused{r}{s}{X}}$.
Note that in the first equation, the transition weight does not depend on $r$, so $r$ does not occur on the left-hand side.
Then let
\begin{align*}
V &= {\unarytable^1}^* \, \unarytable^2 \, {\unarytable^3}^*
\end{align*}
so that $V_{psY,qX}$ is the total weight of runs from configurations of the form $\pdaconfig{\manynullfused{r}{s}Y}{p}$ to configurations of the form $\pdaconfig{X^\nonull}{q}$.

Finally, we remove the unary transitions and modify the non-unary transitions as follows:
\begin{equation*}
\begin{array}{@{}ll@{}}
\text{For each non-unary:} & \text{Replace with:} \\
\pdaEdge{w}{p}{a}{q}{\stackseq}{\manynullfused{r}{s}{X}} & \pdaEdge{w \otimes \unarytable^{1*}_{qsX,tuY}}{p}{a}{t}{\stackseq}{\manynullfused{r}{u}{Y}} \\
& \pdaEdge{w \otimes V_{qsX,tY}}{p}{a}{t}{\stackseq}{Y^\nonull} \\[2ex]
\pdaEdge{w}{p}{a}{q}{\stackseq}{X^\nonull} & \pdaEdge{w \otimes \unarytable^{3*}_{qX,rY}}{p}{a}{r}{\stackseq}{Y^\nonull}
\end{array}
\end{equation*}

Since $V$ can be computed in $\mathcal{O}((|\states|^2|\stackalphabet|)^3)$ time, the asymptotic time complexity of removing nullary and unary transitions is the same when performed directly on the WPDA as compared with converting to a WCFG and removing nullary and unary rules.

\section{Fast Bottom-up Stringsum Algorithm}
\label{sec:appendix-fast-bottom-up-algorithm}

\cref{algo:sped-up-bottom-up-parsing} shows an alternative deductive system for parsing in bottom-up WPDAs. The algorithm that can be derived from this deductive system achieves a runtime improvement by a factor of $|\states|$ and has the same space complexity as Lang's algorithm. This algorithm, however, does not achieve further time and space complexity improvements on the special type of automaton used by Lang.

% \begin{figure*} 
% \begin{equation*}
% \renewcommand{\arraystretch}{2.5}
% \begin{array}{lcll}
% \text{Item form} & \myitem{i}{p}{\stackseq}{j}{q} & \multicolumn{2}{l}{0 \leq i < j \leq n \quad p, q \in \states \quad \stackseq \in \stackalphabet^* \quad |\stackseq| \in \range{1}{2}} \\
% \text{Inference rules} &
% \dfrac{\myitem{i}{p}{\stacksymbol{Y}}{k}{r} \quad \myitem{k}{r}{\stacksymbol{Z}}{j'}{s}}{\myitem{i}{p}{\stacksymbol{Y} \stacksymbol{Z}}{j'}{s}} \\
% &\dfrac{\myitem{i}{p}{\stacksymbol{Y} \stacksymbol{Z}}{j-|a|}{s}}{\myitem{i}{p}{\stacksymbol{X}}{j}{q}} &\pdaEdge{w}{s}{a}{q}{\stacksymbol{Y} \stacksymbol{Z}}{\stacksymbol{X}} & \str\slice{j-|a|}{j} = a \\
% &\dfrac{\myitem{i}{p}{\stacksymbol{Y}}{j-1}{r}}{\myitem{i}{p}{\stacksymbol{X}}{j}{q}}  &\pdaEdge{w}{r}{\str_j}{q}{\stacksymbol{Y}}{\stacksymbol{X}} \\
% &\dfrac{}{\myitem{i}{p}{\stacksymbol{X}}{j}{q}} &\pdaEdge{w}{p}{\str_j}{q}{\varepsilon}{\stacksymbol{X}} & j = i+1 \\
% \text{Goal} & \myitem{0}{\startstate}{\stacksymbol{S}}{n}{\finalstate}
% \end{array}
% \end{equation*}
% \caption{Deductive system corresponding to the alternative sped-up algorithm for stringsums in bottom-up WPDAs in normal form.}
% \label{algo:sped-up-bottom-up-parsing}
% \end{figure*}

\begin{figure} \small
\vspace*{-2ex}
\begin{equation*}
\begin{array}{@{}cc@{}}
\multicolumn{2}{@{}l}{\text{Item form}} \\
\myitem{i}{p}{\stackseq}{j}{q} & \begin{gathered} 0 \leq i < j \leq n; p, q \in \states \\ \stackseq \in \stackalphabet^*; |\stackseq| \in \range{1}{2} \end{gathered} \\
\multicolumn{2}{@{}l}{\text{Inference rules}} \\[2ex]
\dfrac{\myitem{i}{p}{\stacksymbol{Y}}{k}{r} \quad \myitem{k}{r}{\stacksymbol{Z}}{j'}{s}}{\myitem{i}{p}{\stacksymbol{Y} \stacksymbol{Z}}{j'}{s}} \\[3ex]
\dfrac{\myitem{i}{p}{\stacksymbol{Y} \stacksymbol{Z}}{j\compactminus|a|}{s}}{\myitem{i}{p}{\stacksymbol{X}}{j}{q}} &\begin{gathered} \pdaEdge{w}{s}{a}{q}{\stacksymbol{Y}\stacksymbol{Z}}{\stacksymbol{X}} \\ \str\slice{j\compactminus|a|}{j} = a \end{gathered} \\[3ex]
\dfrac{\myitem{i}{p}{\stacksymbol{Y}}{j\compactminus1}{r}}{\myitem{i}{p}{\stacksymbol{X}}{j}{q}} & \pdaEdge{w}{r}{\str_j}{q}{\stacksymbol{Y}}{\stacksymbol{X}} \\ \dfrac{}{\myitem{i}{p}{\stacksymbol{X}}{j}{q}} &\begin{gathered} \pdaEdge{w}{p}{\str_j}{q}{\varepsilon}{\stacksymbol{X}} \\ 
j = i+1 \end{gathered} \\
\multicolumn{2}{@{}l}{\text{Goal}} \\
\myitem{0}{\startstate}{\stacksymbol{S}}{n}{\finalstate}
%\text{Goal} \qquad \myitem{0}{\startstate}{\stacksymbol{S}}{n}{\finalstate}
\end{array}
\end{equation*}
\caption{Deductive system corresponding to the alternative sped-up algorithm for stringsums in bottom-up WPDAs in normal form.}
\label{algo:sped-up-bottom-up-parsing}
\end{figure}

% \begin{figure*}[!ht]
% \begin{align*}
%     Z \Parens{\runsumitem{p}{X}{q}} = &\bigoplus_{a \in \alphabet \cup \set{\varepsilon}} \weight{\pdaEdge{w}{p}{a}{q}{\varepsilon}{\stacksymbol{X}}} \oplus \\
%     &\bigoplus_{\stacksymbol{Y} \in \stackalphabet} \bigoplus_{r \in \states} \bigoplus_{a \in \alphabet \cup \set{\varepsilon}} Z \Parens{\runsumitem{p}{\stacksymbol{Y}}{r}} \otimes \weight{\pdaEdge{w}{r}{a}{q}{\stacksymbol{Y}}{\stacksymbol{X}}} \oplus \\
%     &\bigoplus_{\stacksymbol{Y}, \stacksymbol{Z} \in \stackalphabet} \bigoplus_{r, s \in \states} \bigoplus_{a \in \alphabet \cup \set{\varepsilon}} Z \Parens{\runsumitem{p}{\stacksymbol{Y}}{r}} \otimes Z \Parens{\runsumitem{r}{\stacksymbol{Z}}{s}} \otimes \weight{\pdaEdge{w}{s}{a}{q}{\stacksymbol{YZ}}{\stacksymbol{X}}}
% \end{align*}
% \caption{\Runsum{} $Z \Parens{\runsumitem{p}{X}{q}}$ expressed as a non-linear equation.}
% \label{fig:runsum-recursive-derivation}
% \end{figure*}

\end{document}